%% file: main.tex
\documentclass[preprint,12pt]{elsarticle}
\usepackage{amsmath,amssymb}
\usepackage{subfigure}
\usepackage{xspace}
\usepackage{stmaryrd}
\usepackage{color}
\usepackage[utf8]{inputenc}
\usepackage{listings}
\usepackage{hyperref}
\usepackage{textcomp}

\usepackage{url}
\usepackage{xy}

\usepackage{tikz}
\usetikzlibrary{positioning,arrows,calc,fit}
\usetikzlibrary{backgrounds}
\tikzset{
modal/.style={>=stealth’,shorten >=1pt,shorten <=1pt,auto,node distance=1.5cm,
semithick},
world/.style={circle,draw,minimum size=0.5cm,fill=gray!15},
point/.style={circle,draw,inner sep=0.5mm,fill=black},
reflexive above/.style={->,loop,looseness=7,in=120,out=60},
reflexive below/.style={->,loop,looseness=7,in=240,out=300},
reflexive left/.style={->,loop,looseness=7,in=150,out=210},
reflexive right/.style={->,loop,looseness=7,in=30,out=330}
}

\usepackage{IEEEtrantools}
\usepackage{enumitem}
\usepackage{array}

\usepackage{bm}

\input{macro}

\input{outline}
\usepackage[comments=off,changes=off]{krudces}

\begin{document}
\begin{frontmatter}

\title{Dynamic Epistemic Logic with ASP Updates: Application to Conditional Planning}

\author[inst1]{Pedro Cabalar}
\author[inst2]{Jorge Fandinno}
\author[inst2]{Luis Fari\~{n}as del Cerro}

\address[inst1]
{Department of Computer Science\\
University of Corunna, SPAIN\\
{\tt cabalar@udc.es}
}

\address[inst2]
{Institut de Recherche en Informatique de Toulouse\\ 
Universty of Toulouse, CNRS, FRANCE\\
{\tt \{jorge.fandinno,farinas\}@irit.fr}
}


\begin{abstract}
Dynamic Epistemic Logic (DEL) 
is a family of multimodal logics that has proved to be very successful for epistemic reasoning in planning tasks.
In this logic, the agent's knowledge is captured by modal epistemic operators whereas the system evolution is described in terms of (some subset of) dynamic logic modalities in which actions are usually represented as semantic objects called \emph{event models}.
In this paper, we study a variant of DEL, that we call DEL[ASP], where actions are syntactically described by using an
Answer Set Programming~(ASP) representation instead of event models.
This representation directly inherits high level expressive features like
indirect effects, qualifications, state constraints, defaults, or recursive fluents
that are common in ASP descriptions of action domains.
Besides, we illustrate how this approach can be applied for obtaining conditional plans
in single-agent, partially observable domains where knowledge acquisition may be represented as indirect effects of  actions.
\end{abstract}

\end{frontmatter}

\input{introduction}

\section{Preliminaries}
 
In this section, we provide some background on planning in DEL, planning in ASP, and the ASP extension of epistemic specifications, since these three components will be present in DEL[ASP] up to some degree.
In the case of DEL, we will present a slight generalisation of~\cite{BolanderA11,AndersenBJ12} that admits abstract \emph{updating objects}. 
These objects correspond to event models for standard DEL, which we denote here as DEL[$\cE]$, and will become epistemic specifications for DEL[ASP].
For the case of epistemic logic programs,
we will use a recent logical formalisation~\cite{Cabalar2019faeel} that avoids the problem of self-supported conclusions present in the original semantics~\cite{Gelfond91}.
This logic, called \emph{Founded Autoepistemic Equilibrium Logic} (FAEEL) is a combination of Pearce's \emph{Equilibrium Logic}~\cite{Pearce96}, a well-known logical characterisation of stable models, with Moore's \emph{Autoepistemic Logic} (AEL)~\cite{Moore85}, one of the most representative approaches among modal non-monotonic logics.

\input{del}
\input{asp-planning}
\input{faeel}

\input{deasp}
\input{discussion}

\bibliographystyle{elsarticle-num}
\bibliography{refs,akku,lit,procs}

\end{document}

%% file: macro.tex
\usepackage{amsmath}
\usepackage{amssymb}
\usepackage{xspace}
\usepackage{enumitem}
\usepackage{IEEEtrantools}

\DeclareMathOperator{\bL}{\mathbf{L}}
\DeclareMathOperator{\bO}{\mathbf{O}}

\DeclareMathOperator{\bU}{\mathbf{U}}

\DeclareMathOperator{\bK}{\mathbf{K}}
\DeclareMathOperator{\bM}{\mathbf{M}}
\DeclareMathOperator{\hK}{\mathbf{\hat{K}}}

\def\modelsm{\models_{\text{\rm \tiny m}}}
\def\modelshtm{\models_{\text{\rm \tiny htm}}}

\def\modelsm{\models_{\text{\rm \tiny m}}}

\newcolumntype{L}[1]{>{\raggedright\let\newline\\\arraybackslash\hspace{0pt}}m{#1}}
\newcolumntype{C}[1]{>{\centering\let\newline\\\arraybackslash\hspace{0pt}}m{#1}}
\newcolumntype{R}[1]{>{\raggedleft\let\newline\\\arraybackslash\hspace{0pt}}m{#1}}

\def\cA{\mathcal{A}}

\def\cE{\mathcal{E}}
\def\cN{\mathcal{N}}
\def\cR{\mathcal{R}}
\DeclareMathOperator{\sskip}{\mathrm{skip}}
\DeclareMathOperator{\iif}{\mathrm{if}}
\DeclareMathOperator{\tthen}{\mathrm{then}}
\DeclareMathOperator{\eelse}{\mathrm{else}}

\def\move{\mathtt{move}}
\def\flick{\mathtt{flick}}
\def\taker{\mathtt{take\_right}}
\def\takel{\mathtt{take\_left}}
\def\ttaker{\mathtt{try\_take\_right}}
\def\ttakel{\mathtt{try\_take\_left}}

\newcommand{\den}[1]{\llbracket \, #1 \, \rrbracket}

\newcommand{\previous}{\raisebox{-.5pt}{\Large\textbullet}} 

\newcommand{\LE}[1]{\ensuremath{\mathcal{L}_{\text{E}}(#1)}}
\newcommand{\LEP}{\LE{\at}}

\newcommand{\pre}{\mathit{pre}}
\newcommand{\post}{\mathit{post}}

\newcommand{\cell}[1]{\mathrm{cell}(#1)}
\newcommand{\model}[1]{\mathtt{mwv}(#1)}
\newcommand{\erelation}[2]{\cR_{#1,#2}}

\def\WV{\ensuremath{\text{{\rm\small W\!V}}}}
\def\cS{\ensuremath{\mathcal{S}}}

\def\FAEEL{{\rm FAEEL}\xspace}

\def\uobj{\ensuremath{\bm{\mathsf{O}}}}
\def\at{\ensuremath{\bm{\mathsf{P}}}}
\def\atbi{\ensuremath{\bm{\mathsf{P}^{bi}}}}
\def\Models{\ensuremath{\bm{\mathsf{M}}}}
\def\Worlds{\ensuremath{\bm{\mathsf{W}}}}
\def\pink{\ensuremath{\mathit{pink}}}

\def\GammaPink{\ensuremath{\Gamma_\pink}}

%% file: introduction.tex

\section{Introduction}\label{sec:introduction}

Automated planning is the field of Artificial Intelligence concerned with the generation of strategies to achieve a goal in a given dynamic domain.
A planner usually starts from a formal representation of the domain, a particular instance of the problem and the goal to achieve.
The planner output is some strategy, expressed in terms of actions that cause the state transitions to reach the goal.
The most common situation is that such a strategy is just a sequence of actions called a \emph{plan}.
In Classical Planning~\cite{Ghallab2004} some simplifying restrictions are assumed: the system has a \emph{finite} number of states, the world is \emph{fully observable} and the transition relation is \emph{deterministic} and  \emph{static} (i.e. transitions are only caused by the execution of actions).
However, a rational agent may easily face planning problems that require relaxing these assumptions.
For instance, a robot may not possess all the information about the environment, either because its sensors have a limited scope, or because its actions may have non-deterministic effects that require observation for finding out the real outcome.
Removing the assumptions of determinism and fully observable world naturally leads to two important questions~\cite{TuSB07}: (i) \emph{how does a plan look like in this new context}? and (ii) \emph{how to represent the changes in the agent's knowledge along the plan execution}?

Regarding (i), two new categories of plans have been defined in this context: \emph{conformant plans} and \emph{conditional plans}.
%
%
A \emph{conformant plan} is a sequence of actions that guarantees achieving the goal regardless unknown values of the fluents in the initial situation or the precise effect of the non-deterministic actions.
%
%
If we further allow sensing actions (acquiring knowledge from the environment) then the structure of a sequential plan is not reasonable any more: a \emph{conditional plan} may contain ``if-then-else'' constructs that allow the agent to follow different strategies depending on the knowledge she acquired when executing the plan.
Approaches to both conformant and conditional planning have been broadly studied in the literature~\cite{TuSB07,peot1992conditional,GoldenW96,PryorC96,LoboMT97,blum1997fast,smith1998conformant,Golden98,weld1998extending,rintanen1999constructing,bonet2000planning,castellini2003sat,eiter2003logic,cimatti2004conformant,bryce2006planning,hoffmann2006conformant,lowe2010planning,BolanderA11,pontelli2012answer,AndersenBJ12,Bolander14,KominisG15,KominisG17,CooperHMMR16a,CooperHMMR16b,BolanderEMN18}.

With respect to question (ii), several approaches studied the effects of sensing actions~\cite{GoldenW96,LoboMT97,lowe2010planning,BolanderA11,AndersenBJ12,Bolander14,moore1984formal,thielscher2000representing,son2001formalizing,scherl2003knowledge}.
One prominent line of research is based on \emph{Dynamic Epistemic Logic} (DEL)~\cite{van2007dynamic,van2008semantic}, a multi-modal approach where the agent's knowledge is captured by modal epistemic operators whereas the system evolution is described in terms of (some subset of) dynamic logic~\cite{pratt76a} modalities.
For instance, the DEL expression $[\mathtt{watch}] (\bK rain \vee \bK \sneg rain)$ represents that, after any possible outcome of sensing action $\mathtt{watch}$, the agent knows whether $rain$ holds or not.
Different variants of DEL have been successfully applied to the problem of planning with non-deterministic, partially observable multi-agent domains~\cite{lowe2010planning,BolanderA11,AndersenBJ12,Bolander14,van2005dynamic}.
Although DEL has proved to be very convenient for epistemic reasoning in planning, it shows some important drawbacks when analysed from a Knowledge Representation (KR) viewpoint.
This is because, for representing actions and their effects, DEL uses the so-called \emph{event models}~\cite{BatlagMS98,baltag2004logics}, that inherit some of the expressive limitations of the STRIPS planning language~\cite{STRIPS}.
In particular, event models do not allow some important KR features, like the treatment of indirect effects, action qualifications, state constraints or recursive fluents, that are quite common in modern representation of action theories.

One popular KR formalism that naturally covers these expressive features is \emph{Answer Set Programming} (ASP)~\cite{MT99,Nie99,baral2010}, a well-established paradigm for problem solving and non-monotonic reasoning based on the \emph{stable models} semantics~\cite{GL88,GelfondL91}.
%
%
The use of ASP for classical planning was introduced in~\cite{lifschitz99b,LIFSCHITZ200239}, leading to a methodology adopted by different high-level action languages (see~\cite{lee13} and references there) and, more recently, to a temporal extension of ASP~\cite{CabalarKSS18}.
Besides, there exists a growing list of ASP applications~\cite{ergele16a}, many of them dealing with classical planning problems.
%
%
When moving to conformant planning, though, the application of ASP is still under a exploratory stage.
Most of the attempts in this direction relied on a extension called \emph{epistemic specifications}~\cite{Gelfond91} that incorporate modal constructs (called \emph{subjective literals}) for representing the agent's knowledge.
However, the semantic interpretation of this formalism is still under debate and only some preliminary implementations are still available -- see~\cite{LK18} for a recent survey.
On the other hand, the use of ASP to obtain conditional plans was still an unexplored territory.

In this paper, we study the case of single-agent planning and combine both approaches, DEL and (epistemic) ASP, to exploit the advantages of both formalisms in a single language.
Our proposal, called DEL[ASP], relies on replacing event models by epistemic logic programs.
In that way, the basic event to describe the transition between two epistemic models becomes an ASP epistemic specification, while we keep the same dynamic logic operators for temporal reasoning among transitions.
On the one hand, with respect to DEL, the new approach provides all the expressive power of  ASP for action domains: indirect effects, qualifications, state constraints, defaults, or recursive fluents are directly inherited from ASP.
Moreover, when a classical planning scenario (represented in ASP) becomes partially observable, the new approach allows \emph{keeping the scenario representation untouched}, possibly adding new epistemic rules to describe the effects of sensing actions.
On the other hand, with respect to (non-temporal) epistemic ASP, dynamic operators provide a comfortable way for explicitly representing, and formally reasoning about conformant and conditional plans. 

The rest of the paper is organised as follows.
In the next section, we provide some background on the formalisms that conform our proposal. After these preliminaries, we introduce the formalism of DEL[ASP] and explain its behaviour using some examples. Then, we study the representation of conditional plans in this formalism. Finally, we discuss some related work and conclude the paper.

%% file: del.tex

\subsection{Dynamic Epistemic Logic with Abstract Updating Objects}


Given a set of propositional symbols~$\at$ and a set of updating objects~$\uobj$,
a \emph{(dynamic epistemic) formula}
$\varphi$ is defined according to the following grammar:
\begin{gather*}
\varphi ::=  \bot \mid p \mid \varphi \wedge \varphi \mid \varphi \vee \varphi \mid \varphi \to \varphi \mid \sneg \varphi \mid \bK \varphi \mid [o] \varphi
\end{gather*}
where $p \in \at$ is a proposition and $o \in \uobj$ an updating object.
The modal epistemic operator~$\bK$ represents the (planning) agent's knowledge: formula~$\bK \varphi$ means that ``the agent knows $\varphi$.''
The symbol~``$\sneg\,$'' stands here for classical negation (we reserve the symbol~``$\neg$'' for intuitionistic negation latter on).
A formula~$\varphi$ is called \emph{objective} if the operator $\bK$ does not occur in it.
It is called \emph{subjective} if it has at least some proposition and every proposition is in the scope of~$\bK$.
As usual,
we define the following abbreviations:
$\fF \leftrightarrow \fG \eqdef (\fF \to \fG) \wedge (\fG \to \fF)$,
\ $(\fF \leftarrow \fG) \eqdef (\fG \to \fF)$,
and
\ $\top \eqdef \sneg \bot$.
We also define the dual of $\bK$ as follows: $\hK \varphi \eqdef \sneg\bK\sneg \varphi$.
We keep the Boolean operators $\vee,\wedge,\to,\bot$ and avoid defining ones in terms of the others, since this will not be valid when we use an intuitionistic reading later on.
By $\LEP$ we denote the language containing all dynamic epistemic formulas over~$\at$.

We provide next an abstract semantics that just relies on two basic concepts: \emph{epistemic models} that represent the agent's knowledge; and the \emph{updating evaluation}, a pair of generic functions that describe how updating objects cause transitions among those models.

\begin{definition}[Epistemic Model]
Given a (possibly infinite) set of propositional symbols~$\at$ and a (possibly infinite) set of possible worlds~$\Worlds$,
a \emph{model} is a
triple
$\cM = \tuple{W,\cK,V}$
where
\begin{itemize}
\item $W \subseteq \Worlds$ is a finite set of worlds,
\item $\cK \subseteq W \times W$ is an accessibility relation on $W$, and
\item $V : W \longrightarrow 2^{\at}$ is a valuation function.
\end{itemize}
$D(\cM) = W$
denotes the domain of $M$.
An \emph{epistemic model} is a model where $\cK$ is an equivalence relation: it is further called \emph{information cell} if $\cK$ is an universal relation.
A \emph{belief model} is a model where
\mbox{$\cK = W \times W'$} with $W' = W \setminus \set{w_0}$ for some $w_0 \in W$.
By $\Models$ we denote the set of all possible epistemic models over~$\at$ and $\Worlds$.\qed
\end{definition}

We assume that the modeller coincides with the planning agent (the one whose knowledge is captured by the epistemic models).
This is usually called an \emph{internal} point of view, as opposed to the \emph{external} one where the modeller is a different agent, an omniscient
and external observer who can differentiate the actual world and knows its configuration~\cite{Aucher10,Fagin95}.
Adopting the internal orientation translates in the lack of a designated world (all worlds in the model are equally possible).
A second consequence is that, even for single-agent epistemic models, we cannot replace the equivalence relation by a universal one.

Before going into further technical details, let us introduce the following scenario from~\cite{AndersenBJ12}, which will be our running example throughout the~paper.

\begin{example}\label{ex:pink}
After following carefully laid plans, a thief has almost made it to
her target: the vault containing the invaluable Pink Panther diamond.
Standing outside the vault, she now deliberates on how to get her hands on the
diamond.
She \emph{knows} the light inside the vault is off, and that \emph{the Pink
Panther is on either the right or the left pedestal inside.}
Obviously, the
diamond cannot be on both the right and left pedestal, but nonetheless the
agent may be uncertain about its location.
Note that the thief is perfectly capable of moving in the darkness and take whatever is on top any of the pedestals, but she is not able to know whether the diamond has been taken or not.
It is assumed that there are four possible actions: $\move$, $\flick$, $\takel$ and $\taker$.
The action $\move$ changes the location of the thief from outside the vault~($\sneg v$) to inside the vault~($v$) and vice-versa.
The action $\flick$ turns on the light ($l$).
Furthermore, \emph{if the thief is in the vault ($v$) and the light is on ($l$), the thief can see~($s$) where the Pink Panther is.}
Finally, actions $\takel$ and $\taker$ respectively take the diamond ($d$) from the left ($\sneg r$) or right ($r$) pedestal
if the diamond is in the intended pedestal.
\qed
\end{example}

The set of propositions for the example is $\at=\{v,l,r,s,d\}$.
Figure~\ref{fig:pink.state0-4} depicts three consecutive epistemic models, $\cM_0$, $\cM_1$ and $\cM_2$, respectively corresponding to the initial state of Example~\ref{ex:pink} and the resulting states after performing the sequence of actions $\move$ and then $\flick$.
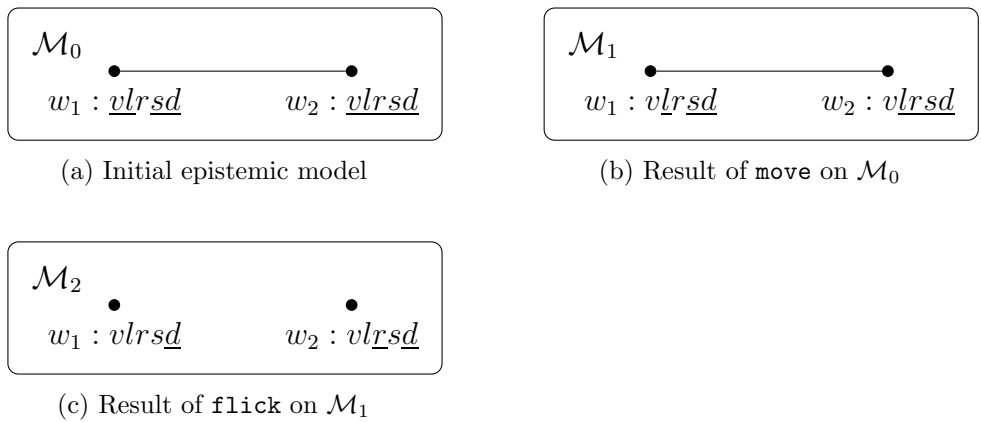
\begin{figure}[htbp]
\begin{center}
\begin{tabular}{ c c c }
\begin{minipage}{0.4\textwidth}
\begin{center}
\begin{tikzpicture}[align=center,node distance=3cm, framed, background rectangle/.style={draw=black, rounded corners}]]
\node[point] (w) [label=below:$w_1 : \underline{vl}r\underline{sd}$, label={[label distance=2mm]175:$\cM_0$} ] {};
\node[point] (v) [label=below:$w_2 : \underline{vlrsd}$, right=of w] {};
\path[-] (w) edge (v);
\end{tikzpicture}
\\
{\footnotesize (a) Initial epistemic model}
\end{center}
\end{minipage}
&\hspace*{0.8cm}&
\begin{minipage}{0.4\textwidth}
\begin{center}
\begin{tikzpicture}[align=center,node distance=3cm, framed, background rectangle/.style={draw=black, rounded corners}]]
\node[point] (w) [label=below:$w_1 : v\underline{l}r\underline{sd}$, label={[label distance=2mm]175:$\cM_1$} ] {};
\node[point] (v) [label=below:$w_2 : v\underline{lrsd}$, right=of w] {};
\path[-] (w) edge (v);
\end{tikzpicture}
\\
{\footnotesize (b) Result of $\move$ on $\cM_0$}
\end{center}
\end{minipage}
\\
\noalign{\vskip 20pt}
\begin{minipage}{0.4\textwidth}
\begin{center}
\begin{tikzpicture}[align=center,node distance=3cm, framed, background rectangle/.style={draw=black, rounded corners}]]
\node[point] (w) [label=below:$w_1 : vlrs\underline{d}$, label={[label distance=2mm]175:$\cM_2$}] {};
\node[point] (v) [label=below:$w_2 : vl\underline{r}s\underline{d}$,right=of w] {};
\end{tikzpicture}
\\
{\footnotesize (c) Result of $\flick$ on $\cM_1$}
\end{center}
\end{minipage}
\end{tabular}
\end{center}
	\caption{Sequence of epistemic models $\cM_0, \cM_1, \cM_2$ that result from actions $\move$ and then $\flick$ starting in the initial state $\cM_0$ of Example~\ref{ex:pink}.
	 }\label{fig:pink.state0-4}
\end{figure}

For improving readability, we represent the world valuations as strings of propositions and underline those that are false.
For instance, the valuation for $w_1$ in Figure~\ref{fig:pink.state0-4}a is depicted as $\underline{vl}r\underline{sd}$ and corresponds to the set of true atoms $\{r\}$, making false all the rest.
The initial model, $\cM_0$ represents the triple $\tuple{W_0,\cK_0,V_0}$ where we have two worlds $W_0=\{w_1,w_2\}$ universally connected, that is, $\cK_0 = W_0 \times W_0 = \set{ (w_1,w_2), (w_2, w_1), (w_1,w_1), (w_2,w_2)}$ with valuations $V_0(w_1) = \set{r}$ and $V_0(w_2) = \emptyset$. 
This means that the agent knows that the thief is outside the vault ($\underline{v}$), the light is off ($\underline{l}$), she cannot see where the Pink Panther is ($\underline{s}$) and she does not have the diamond~($\underline{d}$).
The two connected worlds reveal that she does not know whether the Pink Panther is on the right pedestal (world $w_1$) or on the left one (world $w_2$).
The epistemic model $\cM_1$ in Figure~\ref{fig:pink.state0-4}b reflects the same configuration, with the only exception that now the thief is in the vault ($v$), as a result of moving inside.
That is, \mbox{$\cM_1 = \tuple{W_0,\cK_0,V_1}$}
with $V_1$ satisfying $V_1(w_1) = \set{v,r}$ and $V_0(w_2) = \set{v}$.
Finally, $\cM_2$ shows the epistemic model that results from performing action $\flick$ on $\cM_1$.
In this third model, two relevant changes occur: first, $l$ and $s$ became true in both worlds, since flicking turns on the light, and then, the thief can see the interior of the vault. 
Second, more importantly, the two worlds became completely disconnected, something that reveals that the agent has now complete knowledge of the world configuration in the two possible cases, $w_1$ and $w_2$.
Formally, we have that
$\cM_2 = \tuple{W_0,\cK_2,V_2}$
with
$\cK_2 = \{(w_1,w_1), (w_2,w_2)\}$
and with $V_2$ satisfying $V_2(w_1) = \set{v,l,r,s}$ and $V_2(w_2) = \set{v,l,s}$.

As we can see, the accessibility relation needs not be universal: for instance, we had $(w_1,w_2) \not\in \cK_2$ in $\cM_2$ above.
In general, when $(w,w') \not\in \cK$ we say that the two worlds are \emph{indistinguishable at plan-time}, given epistemic model $\cM=\tuple{W,\cK,V}$ with $\{w_1,w_2\}\subseteq W$.
In the example, model $\cM_2$ tells us that, before executing the plan, the agent cannot tell which of the two possibilities represented by worlds $w_1$ (diamond on the right) and $w_2$ (diamond on the left) will correspond to the actual world.
However, once she executes the plan, she will acquire that knowledge in both cases: $w_1$ and $w_2$ are disconnected, so uncertainty in the agent's knowledge is completely removed.
As a result, at that point, she will be able to make a decision whether she should perform the $\taker$ or $\takel$ action.

If, on the contrary, two worlds $w, w'$ are connected $(w,w') \in \cK$ in some epistemic model, we say that they are \emph{indistinguishable at run-time}.
This expresses a higher degree of uncertainty, since the agent has no way to tell which world corresponds to ``reality'' either before or during the plan execution.
For instance, at model $\cM_1$ we have $(w_1,w_2) \in \cK_1$ meaning that, during the plan execution, the agent will not be able to decide (at that point) whether actions $\taker$ or $\takel$ will work.

Until now, we have only described the information captured by epistemic models and presented an example with transitions among them, but did not specify how those transitions were obtained.
For that purpose, we will only assume, by now, the existence of a pair of generic functions called \emph{updating evaluation} defined below.

\begin{definition}[Updating evaluation]
Given a set of models $\Models$ over set of worlds $\Worlds$ and a set of updating objects $\uobj$,
an \emph{updating evaluation}
is a pair $\tuple{\otimes,\cR}$
of partial functions
\mbox{$\otimes : \Models \times \uobj \longrightarrow \Models$}
and
\mbox{$\cR : \Models \times \uobj \longrightarrow 2^{\Worlds \times \Worlds}$}
satisfying
$\cR(\cM,o) \subseteq D(W) \times D(\otimes(\cM, o))$.\qed
\end{definition}

Function $\otimes$ takes an initial model $\cM$ and some updating object $o$ and provides a successor model $\cM'=\otimes(\cM,o)$.
We will usually write $\otimes$ in infix notation and assume that it is left associative 
so that $\cM \otimes o$ stands for $\otimes(\cM,o)$
and
$\cM \otimes o_1 \otimes o_2$ stands for $(\cM \otimes o_1) \otimes o_2$.
Relation $\cR(\cM,o)$ matches worlds from the initial model $\cM$ and its successor model $\cM'=\cM\otimes o$.
We will also write $\cR_{\cM,o}$ instead of $\cR(\cM,o)$.

At this point, we have all the elements for defining the satisfaction of dynamic epistemic formulas.

\begin{definition}[Satisfaction]
\label{def:del.sat}
Let
\mbox{$\tuple{\otimes,\cR}$}
be an updating evaluation.
Then, given an epistemic model~\mbox{$\cM = \tuple{W,\cK,V}$},
satisfaction of formulas is given by the following recursive definition:
\begin{itemize}
\item $\cM, w \not\models \bot$,
\item $\cM, w \models p$ iff $p \in V(w)$,
\item $\cM, w \models \varphi_1 \wedge \varphi_2$ iff $\cM, w \models \varphi_1$ and $\cM, w \models \varphi_2$,
\item $\cM, w \models \varphi_1 \vee \varphi_2$ iff $\cM, w \models \varphi_1$ or $\cM, w \models \varphi_2$,

\item $\cM, w \models \varphi_1 \to \varphi_2$ iff $\cM, w \not\models \varphi_1$ or $\cM, w \models \varphi_2$,

\item $\cM, w \models \sneg \varphi$ iff $\cM, w \not\models \varphi$,

\item $\cM, w \models \bK \varphi$ iff $\cM, w' \models \varphi$ for all $w'$ with $(w,w') \in \cK$, and

\item $\cM, w \models [o] \varphi$ iff
$\cM \otimes o$ and $\erelation{\cM}{o}$ are defined and
$\cM \otimes o, w' \models \varphi$\\\hspace*{70pt}holds for all $w'$ with $(w,w') \in \erelation{\cM}{o}$.
\qed
\end{itemize}

\end{definition}

As usual, we write
\mbox{$\cM \models \varphi$} iff
\mbox{$\cM,w\models\varphi$} for every world
\mbox{$w \in W$}.
Furthermore, given a theory
$\Gamma$, we write
\mbox{$\cM \models \Gamma$} iff
\mbox{$\cM \models \varphi$} for every formula
\mbox{$\varphi \in \Gamma$}.
We say that theory $\Gamma$ \emph{entails} formula $\psi$, also written 
\mbox{$\Gamma \models \psi$}, iff \mbox{$\cM \models \Gamma$} implies 
\mbox{$\cM \models \psi$} for any epistemic model
\mbox{$\cM \in \Models$}.

It is easy to see that the semantics for the dynamic-free fragment of the language (i.e., without $[\cdot]$ operator) corresponds to modal logic S5 (see~\cite{halpern1986reasoning} for instance).

\subsection{Dynamic Epistemic Logic with Event Model Updates: {\rm DEL[\ensuremath{\cE}]}}

Let us now see how these definitions apply to the case in which updating objects are event models~\cite{BatlagMS98}.
The following is an adaptation of the definition from~\cite{AndersenBJ12}.
A first peculiarity of event models is that, when making an update $\cM \otimes o=\cM'$, the resulting epistemic model $\cM'$ uses world names of the form $(w,e)$ where $w$ is a world from the updated epistemic model $\cM$ and $e$ is a world (or event) from the event model~$o$.
For this reason, along this section, we assume that the global set of available world names $\Worlds$ is closed under formation of pairs.
In other words, $\Worlds$ satisfies $(w,w') \in \Worlds$ for all $w,w' \in \Worlds$.
For instance, given a unique ``atomic'' world name $w_0 \in \Worlds$, the set $\Worlds$ would contain infinitely many pairs $(w_0,w_0), ((w_0,w_0),w_0), (w_0,(w_0,w_0)), \dots$ and so on.

\begin{definition}[Event Model]
An \emph{event model} over $\at$ and $\Worlds$ is a quadruple $\cE = \tuple{E,\cK,\pre,\post}$
where
\begin{itemize}
\item $E \subseteq \Worlds$ is a finite set of worlds called \emph{events},
\item $\pre : E \longrightarrow \LEP$ assigns to each event a precondition, and
\item $\post : E \longrightarrow (\at \longrightarrow \LEP)$ assigns to each event a postcondition, for some propositions in $\at$.
\item $\cK \subseteq E \times E$
\end{itemize}
By $D(\cE) = E$ we denote the domain of $\cE$.
A pair $\tuple{\cE,e}$ with $e \in E$ is called a \emph{pointed event model}.\qed
\end{definition}

\begin{definition}[Event updating evaluation]
Let
\mbox{$\cM = \tuple{W,\cK,V}$} be an epistemic model and
\mbox{$\cE = \tuple{E,\hat\cK,\pre,\post}$}
 an event model, both over $\at$ and~$\Worlds$.
The \emph{product update} $\cM \otimes \cE \eqdef \tuple{W',\cK',V'}$ is another epistemic model where
\begin{itemize}
\item $W' = \setm{ (w,e) \in W \times E}{ \cM,w \models \pre(e)} \subseteq \Worlds$ is a set of worlds,
\item $\cK' = \setm{((w_1,e_1),(w_2,e_2)) \in W' \times W' }{ (w_1,w_2) \in \cK \text{ and } (e_1,e_2) \in \hat\cK }$,
\item $V'((w,e)) = \setm{p \in \at}{ \cM,w \models\post(e)(p) }$ for every $(w,e) \in W'$,
\end{itemize}

Given a pointed event model $\tuple{\cE,e}$,
the \emph{event updating evaluation} is a pair $\tuple{\otimes,\cR}$
with
\begin{itemize}
\item $\cM \otimes \tuple{\cE,e} \eqdef \cM \otimes \cE$
\item $\cR(\cM,\tuple{\cE,e}) \eqdef \setm{(w,w') \in W \times W' }{ w' = (w,e) }$.\qed
\end{itemize}
\end{definition}

For simplicity,
we will usually write $[\cE,e] \varphi$ instead of $[\tuple{\cE,e}] \varphi$.
We will also use the following shorthands
\begin{gather*}
[\cE] \varphi \ \eqdef \ \bigwedge_{e \in D(\cE)} [\cE,e] \varphi
\hspace{2cm}
\tuple{\cE} \varphi \ \eqdef \ \sneg\,[\cE]\sneg\varphi
\end{gather*}
The following result shows that, indeed, the semantics described above coincides with the semantics from~\cite{AndersenBJ12} for the case of event models.

\begin{proposition}
Let $\cM$ be an epistemic model, $w \in D(\cM)$ be a world in $\cM$,
$\tuple{\cE,e}$ be pointed event model and $\varphi \in \LEP$ be a formula. Then,
\begin{itemize}
\item $\cM,w \models [\cE,e] \varphi$ iff $\cM,w \models \pre(e)$ implies $\cM \otimes \cE, (w,e) \models \varphi$,


\end{itemize}
\end{proposition}

\begin{proof}
By definition, we have $\cM,w \models [\cE,e]$
\\iff
$\cM \otimes \tuple{\cE,e}, w' \models \varphi$ for all $w'$ with $(w,w') \in \erelation{\cM}{\tuple{\cE,e}}$
\\iff
$\cM \otimes \cE, w' \models \varphi$ for all $w'$ with $(w,w') \in \erelation{\cM}{\tuple{\cE,e}}$.
\\
Note now that, by definition, we have either
$\erelation{\cM}{\tuple{\cE,e}} \cap (\set{w} \times D(\cM \otimes \cE)) = \set{(w,(w,e))}$
or $\erelation{\cM}{\tuple{\cE,e}} \cap (\set{w} \times D(\cM \otimes \cE)) = \emptyset$.
Hence, the above holds
\\iff
$\cM \otimes \cE, (w,e) \models \varphi$ or $\erelation{\cM}{\tuple{\cE,e}} \cap (\set{w} \times D(\cM \otimes \cE)) = \emptyset$,
\\iff
$\cM \otimes \cE, (w,e) \models \varphi$ or $\cM,w \not\models \pre(e)$
\\iff
$\cM,w \models \pre(e)$ implies $\cM \otimes \cE, (w,e) \models \varphi$.
\end{proof}

The following result helps undertanding the semantics of (non-pointed) event models.

\begin{proposition}
Let $\cM$ be an epistemic model, $w \in D(\cM)$ be a world in $\cM$,
$\cE$ be an event model and $\varphi \in \LEP$ be a formula. Then,
\begin{itemize}

\item $\cM,w \models [\cE] \varphi$ iff $\cM \otimes \cE, (w,e) \models \varphi$ for every $e \in D(\cE)$ such that $\cM,w \models \pre(e)$,

\item $\cM,w \models \tuple{\cE} \varphi$ iff $\cM,w \models \pre(e)$ and $\cM \otimes \cE, (w,e) \models \varphi$ for some $e \in D(\cE)$.
\end{itemize}
\end{proposition}

\begin{proof}
The second statement follows directly from its definitions.
For the third, we have
$\cM,w \models \tuple{\cE} \varphi$
iff
$\cM,w \models \sneg\,[\cE] \sneg \varphi$
iff
$\cM,w \not\models [\cE] \sneg \varphi$
\\iff
$\cM,w \models \pre(e)$ does not imply $\cM \otimes \cE, (w,e) \models \sneg\varphi$ for some $e \in D(\cE)$
\\iff
$\cM,w \models \pre(e)$ and $\cM \otimes \cE, (w,e) \not\models \sneg\varphi$ for some $e \in D(\cE)$
\\iff
$\cM,w \models \pre(e)$ and $\cM \otimes \cE, (w,e) \models \varphi$ for some $e \in D(\cE)$.
\end{proof}

\begin{figure}
\small
\begin{center}
\begin{minipage}{0.45\textwidth}
\begin{center}
\begin{tikzpicture}[align=center,node distance=2.5cm, framed, background rectangle/.style={draw=black, rounded corners}]]
\node[point] (w) [label=below:$e_1 : \tuple{v \wedge \sneg d, \set{ d \mapsto \sneg r }}$ ] {};
\end{tikzpicture}
\\
(a) $\takel$
\end{center}
\end{minipage}
\hspace{0cm}
\begin{minipage}{0.45\textwidth}
\begin{center}
\begin{tikzpicture}[align=center,node distance=2.5cm, framed, background rectangle/.style={draw=black, rounded corners}]]
\node[point] (w) [label=below:$e_1 : \tuple{v \wedge \sneg d, \set{ d \mapsto r }}$ ] {};
\end{tikzpicture}
\\
(b) $\takel$
\end{center}
\end{minipage}
\\[20pt]
\begin{minipage}{\textwidth}
\begin{center}
\begin{tikzpicture}[align=center,node distance=5cm, framed, background rectangle/.style={draw=black, rounded corners}]]
\node[point] (w) [label=above:$e_1 : \tuple{v \wedge r, \set{ l \mapsto \top, s \mapsto \top }}$ ] {};
\node[point] (v) [label=below:$e_2 : \tuple{v \wedge \sneg r, \set{ l \mapsto \top, s \mapsto \top }}$ ,right=of w ] {};
\end{tikzpicture}
\\
(c) $\flick$
\end{center}
\end{minipage}
\\[20pt]
\begin{minipage}{\textwidth}
\begin{center}
\begin{tikzpicture}[align=center,node distance=4cm, framed, background rectangle/.style={draw=black, rounded corners}]]
\node[point] (w) [label=above:$e_1 : \tuple{v \vee \sneg l, \set{ v \mapsto \sneg v }}$ ] {};
\node[point] (v) [label=below:$e_2 : \tuple{\sneg v \wedge l \wedge r, \set{ v \mapsto \sneg v, s \mapsto \top }}$ ,right=of w ] {};
\node[point] (y) [label=above:$e_2 : \tuple{\sneg v \wedge l \wedge \sneg r, \set{ v \mapsto \sneg v, s \mapsto \top }}$ ,right=of v ] {};
\end{tikzpicture}
\\
(d) $\move$
\end{center}
\end{minipage}
\end{center}
	\caption{Event models corresponding to the actions of Example~\ref{ex:pink}.
	The first element of the pair assigned to each world corresponds to its preconditions while the second one corresponds to its postconditions.
	For instance, in (a) the precondition is $v \wedge \sneg d$ and the postcondition $\set{d \mapsto \sneg r}$. This postcondition means that, in the next state, $d$ takes the value that formula $\sneg r$ had in the previous state.
	 }
	 \label{fig:pink.event.models}
\end{figure}
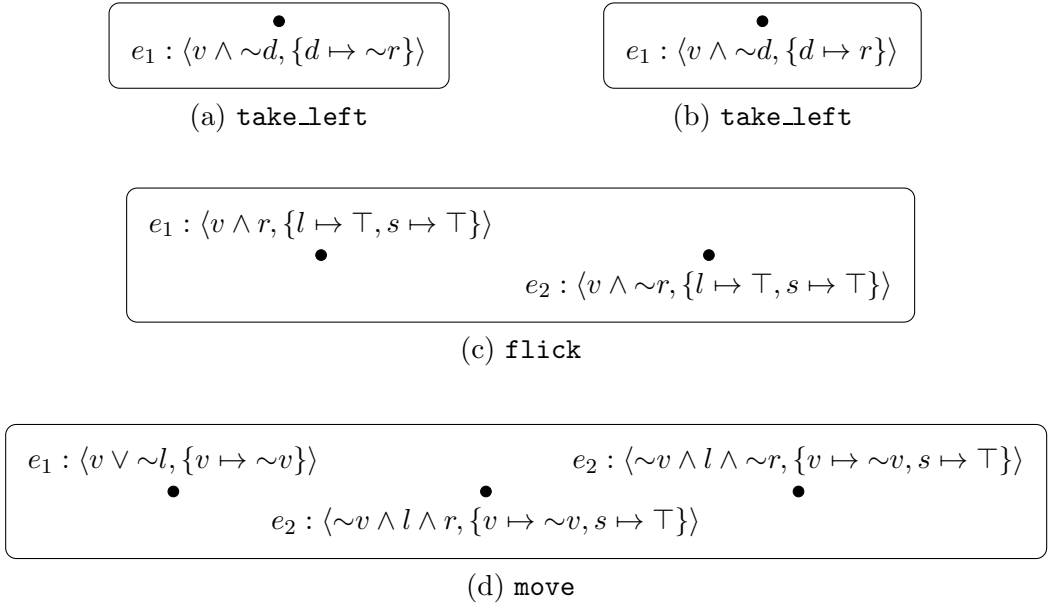
Going back to our running example,
Figure~\ref{fig:pink.event.models}
depicts the event models corresponding to the actions of Example~\ref{ex:pink}.
For instance, Figure~\ref{fig:pink.event.models}a depicts the event model of the action $\takel$ with a single event $e_1$ whose precondition is $v \wedge \sneg d$ and whose postcondition $\set{d \mapsto \sneg r}$ states that, in the next state, $d$ is assigned  the truth value that formula $\sneg r$ had in the previous state.
More interestingly, Figure~\ref{fig:pink.event.models}c depicts the event model of the action $\flick$ which has two events $e_1$ and $e_2$ with the same postcondition but different preconditions.
The precondition of $e_1$ makes it applicable when the thief is in the vault and the diamond is on the right pedestal while the precondition of $e_2$ is analogous but for the left pedestal.\footnote{Note how we must specify the diamond's location in both preconditions, although the only real physical requirement for $\flick$ is being inside the vault. This need for specifying unrelated preconditions may obviously become a representational problem.}
In this sense, when the action $\flick$ is executed in the epistemic model $\cM_1$ (Figure~\ref{fig:pink.state0-4}b), 
it follows that only $w_1$ satisfies the preconditions of $e_1$ and only $w_2$ satisfies the preconditions of $e_2$.
As a result, we can see that $\cM_1 \otimes \flick$ has two worlds, that is,
$D(\cM_1 \otimes \flick) = \set{(w_1,e_1),(w_2,e_2)}$.
Furthermore, since events $e_1$ and $e_2$ are disconnected, we also get that
worlds $(w_1,e_1)$ and $(w_2,e_2)$ are disconnected.
In fact, the epistemic model $\cM_1 \otimes \flick$ is isomorphic to the model $\cM_2$ depicted in Figure~\ref{fig:pink.state0-4}c and can be obtained just by renaming each world $w_i$ in $\cM_2$ as $(w_i,e_i)$.

Note that the existence of two disconnected events in the action $\flick$ encodes the observation that happens when the light is turned on, that is, the agent obtains the knowledge about the actual place of the diamond.
\begin{figure}
\small
\begin{minipage}{0.5\textwidth}
\begin{center}
\begin{tikzpicture}[align=center,node distance=2cm, framed, background rectangle/.style={draw=black, rounded corners}]]
\node[point] (w) [label=above:$e_1 : \tuple{v \wedge r, \set{ l \mapsto \top, s \mapsto \top }}$ ] {};
\node[point] (v) [label=below:$e_2 : \tuple{v \wedge \sneg r, \set{ l \mapsto \top, s \mapsto \top }}$ ,right=of w ] {};
\path[-] (w) edge (v);
\end{tikzpicture}
\\
(a) $\flick'$
\end{center}
\end{minipage}
\hspace{1cm}
\begin{minipage}{0.4\textwidth}
\begin{center}
\begin{tikzpicture}[align=center,node distance=3cm, framed, background rectangle/.style={draw=black, rounded corners}]]
\node[point] (w) [label=below:$w_1 : vlrs\underline{d}$, label={[label distance=2mm]175:$\cM_2'$}] {};
\node[point] (v) [label=below:$w_2 : vl\underline{r}s\underline{d}$,right=of w] {};
\path[-] (w) edge (v);
\end{tikzpicture}
\\
(b)
\end{center}
\end{minipage}
	\caption{(a) Event model corresponding to a variation of the action $\flick$ of Example~\ref{ex:pink} without observing the position of the diamond. (b) Epistemic model obtained after executing the action $\flick'$ in the model $\cM_1$.
	 }
	 \label{fig:pink.event.models'}
\end{figure}
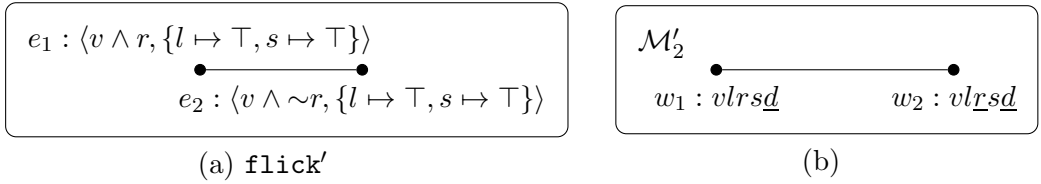
For instance, if we consider the action $\flick'$ depicted in Figure~\ref{fig:pink.event.models'}a, and obtained from the action $\flick$ by connecting events $e_1$ and $e_2$,
we can see that $\cM_1 \otimes \flick'$ is isomorphic to the epistemic model $\cM_2'$ depicted in Figure~\ref{fig:pink.event.models'}b.
Model $\cM_2'$ only differs from $\cM_2$ in that worlds $w_1$ and $w_2$ are now connected, revealing that the agent cannot tell where is the diamond. In other words, $\flick'$ encodes the same ontic changes in the world than $\flick$ but does not capture the agent's observation about the position of the diamond.
\begin{figure}
\begin{center}
\begin{minipage}{0.4\textwidth}
\begin{center}
\begin{tikzpicture}[align=center,node distance=3cm, framed, background rectangle/.style={draw=black, rounded corners}]]
\node[point] (w) [label=below:$w_1 : vlrsd$, label={[label distance=2mm]175:$\cM_3$}] {};
\node[point] (v) [label=below:$w_2 : vl\underline{r}sd$,right=of w] {};
\end{tikzpicture}
\\
(a)
\end{center}
\end{minipage}
\hspace{2cm}
\begin{minipage}{0.4\textwidth}
\begin{center}
\begin{tikzpicture}[align=center,node distance=3cm, framed, background rectangle/.style={draw=black, rounded corners}]]
\node[point] (w) [label=below:$w_1 : \underline{v}lrsd$, label={[label distance=2mm]175:$\cM_4$}] {};
\node[point] (v) [label=below:$w_2 : \underline{v}l\underline{r}sd$,right=of w] {};
\end{tikzpicture}
\\
(b)
\end{center}
\end{minipage}
\hspace*{6.65cm}
\end{center}
	\caption{Epistemic models representing (a) the state corresponding to execution of action $\takel$ or $\taker$ according to the agent's knowledge about the position of the diamond and (b) the result of executing $\move$ on $\cM_3$ in (a).
	}\label{fig:pink.state0-4b}
\end{figure}
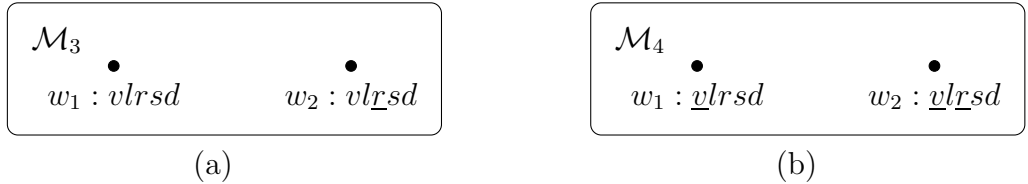

Finally, Figure~\ref{fig:pink.state0-4b}a (model~$\cM_3$) corresponds to an state where the thief is inside the vault with the diamond in her possession.
Intuitively, this model represent the result of executing action $\takel$ or $\taker$ according to the agent's knowledge about the position of the diamond, whereas Figure~\ref{fig:pink.state0-4b}b represents model $\cM_4 = \cM_3 \otimes \move$, that is, the result of moving (outside the vault) afterwards.

%% file: asp-planning.tex

\subsection{Planning in Answer Set Programming}

In this subsection, we informally describe the ASP methodology for representing problems of classical planning: for a more formal approach we refer to~\cite{LIFSCHITZ200239,cakascsc18a}.

Our purpose is, by now, merely introductory, trying to illustrate the main representational features of ASP planning that are relevant for the current discussion. 
For this reason, we delay the introduction of a formal semantics for later on, when epistemic ASP is introduced.

ASP specifications or \emph{logic programs} are sets of rules of the form:
\begin{eqnarray*}
a \leftarrow b_1,\dots,b_n, \Not c_1, \dots, \Not c_m 
\end{eqnarray*}
where $\leftarrow$ is a reversed implication, so its left hand side is called the rule \emph{head} and its right hand side receives the name of rule \emph{body}.
The rule head $a$ is a proposition or the symbol $\bot$: when the latter happens, the rule is a \emph{constraint} forbidding that the body holds.
The elements in the body ($b_i$ and $\Not c_j$) are called \emph{literals}, where $b_i$ and $c_j$ are propositions.
The ordering among literals is irrelevant: in fact, commas just stand for conjunctions.
Operator $\Not$ represents \emph{default negation}: we read $\Not c_j$ as ``there is no evidence on $c_j$'' or ``there is no way to prove $c_j$''. We will also use non-deterministic rules of the form:
\begin{eqnarray*}
m \ \{a_1; \dots; a_k \} \ n \leftarrow Body
\end{eqnarray*}
where $m,n\geq 0$ meaning that, when $Body$ holds, we can arbitrarily add a subset of atoms from $\{a_1,\dots,a_k\}$ of cardinality $c$ with $n \leq c \leq m$.
ASP allows a second kind of negation called \emph{strong} or \emph{explicit} that we will represent $\sneg p$.
From a practical point of view, we can assume that ``$\sneg p$'' is a new kind of atom and that models cannot make $p$ and $\sneg p$ true simultaneously.

For a simple representation of rules describing transitions we partly adopt the syntax of~\cite{cakascsc18a} and assume that, for each proposition $p$, we handle a second atom ``$\previous p$'' that stands for $p$ at the immediately \emph{previous} situation.
In temporal ASP, actions are represented as regular propositions in $\at$: the rest of non-action propositions in $\at$ are called \emph{fluents}.

Taking all these considerations, the behaviour of action $\takel$ can be encoded in ASP as the following three rules:
\begin{IEEEeqnarray}{c C l}
d &\leftarrow& \takel
	\label{eq:takel.postcondition}
\\
\bot &\leftarrow& \takel, \sneg\previous v
	\label{eq:takel.pre1}
\\
\bot &\leftarrow& \takel, \previous r
	\label{eq:takel.pre2}
\end{IEEEeqnarray}
where~\eqref{eq:takel.postcondition} describes its direct effect (grasping the diamond) whereas the other two rules describe the preconditions: \eqref{eq:takel.pre1} forbids executing $\takel$ when the thief was not in the vault and~\eqref{eq:takel.pre2} forbids its execution when the diamond is in the right pedestal.
Analogously, the following three rules encode the action $\taker$:
\begin{IEEEeqnarray}{c C l}
d &\leftarrow& \taker
	\label{eq:taker.postcondition}
\\
\bot &\leftarrow& \taker, \sneg \previous v
	\label{eq:taker.pre1}
\\
\bot &\leftarrow& \taker, \sneg \previous r
	\label{eq:taker.pre2}
\end{IEEEeqnarray}
Similarly, actions $\flick$ and $\move$ are respectively represented by the rules:
\begin{IEEEeqnarray}{r C l}
l &\leftarrow& \flick
	\label{eq:flick.post}
\\
\bot &\leftarrow& \flick, \sneg \previous v
	\label{eq:flick.pre}
\\
v &\leftarrow& \move, \sneg\previous v
	\label{eq:move.post1}
\\
\sneg v &\leftarrow& \move, \previous v
	\label{eq:move.post2}
\end{IEEEeqnarray}
Rule~\eqref{eq:flick.post} states the postcondition of $\flick$, that is, the light is turned on, while rule~\eqref{eq:flick.pre} states its precondition, that is, we forbid its execution when being outside vault.
Rules~\eqref{eq:move.post1} and~\eqref{eq:move.post2} together state the postconditions of $\move$: its execution just flips the truth value of $v$.

To illustrate the use of indirect effects, we can just assert that seeing the diamond~($s$) just depends on being in the vault~($v$) with the light on~($l$), regardless of the actions that have been performed to reach that situation.
This is naturally captured by the single ASP rule:
\begin{IEEEeqnarray}{c C l}
s &\leftarrow& v, l
	\label{eq:see}
\end{IEEEeqnarray}

Default negation allows a natural representation of the inertia law, that is, a fluent normally remains unchanged, unless there is evidence on the contrary.
We divide the set of fluents into \emph{inertial} 
$\cF^\cI = \set{v,l,r,d}$ and \emph{non-inertial} fluents
$\cF^\cN = \set{s}$.
Inertia is then guaranteed by the pair of rules:
\begin{IEEEeqnarray}{r C l}
f &\leftarrow& \previous f, \Not\!\sneg f
	\label{eq:pink.6}
\\
\sneg f &\leftarrow& \sneg\previous f, \Not f
	\label{eq:pink.7}
\end{IEEEeqnarray}
for each inertial fluent~$f \in \cF^\cI$.
In our running example,
the fluent~($s$) is considered false by default, that is, the following rule:
\begin{IEEEeqnarray}{c C l}
\sneg s &\leftarrow& \Not s
	\label{eq:pink.8}
\end{IEEEeqnarray}
sating that, unless~($s$) is proved, we should consider that its explicit negation~($\sneg s$) is added to the set of conclusions.

If we consider now the following simplification of Example~\ref{ex:pink} where the value of all fluents in the initial situation are known, we can use ASP to obtain a plan to achieve the thief's goal.

\begin{examplecont}{ex:pink}\label{ex:pink.simple}
Consider now the case where the thief is outside the vault ($\sneg v$) and already \emph{knows the Pink Panther is inside the vault on the right ($r$) pedestal}.\qed
\end{examplecont}

Listing~\ref{list:pink.simple} shows the full encoding representing Example~\ref{ex:pink.simple} in the syntax of the ASP solver {\ttfamily telingo}.\footnote{\url{https://github.com/potassco/telingo}.}
In this syntax, $\leftarrow$ is represented as {\tt :-}, $\previous p$ as {\tt \textquotesingle p} and $\sneg p$ as {\tt -p}.
By copying that encoding into a file called {\ttfamily pink.lp} and executing ``{\ttfamily telingo pink.lp}'' we can obtain a plan for this example.

\lstset{ 
  basicstyle=\ttfamily\small,
  breaklines=false,                
  captionpos=b,                    
  frame=single,	                   
  keepspaces=true,                 
  keywordstyle=\bf\color{blue},       
  morekeywords={always, dynamic, initial, final, not},            
  otherkeywords={\#program, \#show},
  numbers=left,                    
  numbersep=10pt,                   
  numberstyle=\tiny, 
  rulecolor=\color{black},         
  showspaces=false,                
  showstringspaces=false,          
  showtabs=false,                  
  stepnumber=5,                    
  tabsize=2,	                   
  upquote=true
}
\begin{lstlisting}[float,
  caption={Program corresponding to Example~\ref{ex:pink.simple} in the syntax of the solver {\ttfamily telingo}.},
  label={list:pink.simple}
]
#program dynamic.
d :- take_left.
:- take_left, not 'v.
:- take_left, not -'r.
d :- take_rigth.
:- take_rigth, not 'v.
:- take_rigth, not 'r.
l :- flick.
:- flick, not 'v.
v :- move, -'v.
-v :- move, 'v.
s :- v,l.
 v :- 'v, not -v.
-v :- -'v, not v.
 d :- 'd, not -d.
-d :- -'d, not d.
 l :- 'l, not -l.
-l :- -'l, not l.
 r :- 'r, not -r.
-r :- -'r, not r.
{take_left; take_rigth; flick; move}1.

#program always.
-s :- not s.

#program initial.
-v. r. -d. -l.

#program final.
:- not -v.
:- not d.

#show take_left/0.
#show take_rigth/0.
#show flick/0.
#show move/0.
\end{lstlisting}

As we said before, an important difference between ASP and event models is the treatment of indirect effects. In the example, note how $s$ was captured by the ASP rule \eqref{eq:see}, which only depends on other fluents ($v$ and $l$) but does not refer to the original actions that caused their values.
There is no flexible way to express this feature when using event models: the value of $s$ must be expressed as a direct effect of actions $\flick$ and $\move$, that respectively determine the values of $l$ and $v$. If new actions could alter the values of those fluents, directly or indirectly, then their effect on $s$ should also be included in their post-conditions. This is, in fact, an instance of the well-known \emph{ramification problem}~\cite{KAu86}.

The ramification problem may also occur for epistemic effects, if we are not careful enough for their treatment. For instance, the encoding of Example~\ref{ex:pink} in~\cite{AndersenBJ12} did not use our fluent $s$ (where the ramification is evident), but transferred the problem to the epistemic effect of \emph{knowing} the position of the diamond ($r$).
Again, this epistemic effect must be specified as a direct consequence of some action, something that does not always seem reasonable.

In the rest of the paper, we develop an extension of DEL and ASP, that we denote DEL[ASP] where the ontic and epistemic effects of actions can be described both in a   direct or indirect way, as needed.
In particular, in DEL[ASP], the observation of the diamond position when the thief is in the illuminated vault can be expressed by the following rule analogous to
\eqref{eq:see}:
\begin{IEEEeqnarray}{c C l}
\bO r &\leftarrow& v, l
	\label{eq:obs}
\end{IEEEeqnarray}
Here, $\bO r$ is a new construct whose intuitive meaning is that ``the actual value of the fluent~$r$ is observed (by the agent).''
Note that we just replace the fluent~$s$ in \eqref{eq:see}, whose intuitive meaning is that the agent sees the position of the diamond, by this new construct $\bO r$, which makes this observation affect the agent's beliefs.

%% file: faeel.tex

\subsection{Epistemic Logic Programs}

As explained before, we will use FAEEL for the interpretation of epistemic specifications, the epistemic extension of ASP. FAEEL inherits both the ASP capabilities for knowledge representation and the AEL capabilities for introspective reasoning.
For the sake of coherence, we adapt the definitions of~\cite{Cabalar2019faeel} to the use of Kripke structures.
We also add strong negation~\cite{Pearce96,nelson1949} to FAEEL, which for simplicity, is restricted to be used only in front of atoms, something that suffices for the goals of this paper and is usual in the ASP literature~\cite{GelfondL91,LifschitzTT99}.

\emph{Autoepistemic formulas} are defined according to the following grammar:
\[
\fF ::= \bot \mid p \mid \sneg p \mid \fF_1 \wedge \fF_2 \mid \fF_1 \vee \fF_2 \mid \fF_1 \to \fF_2 \mid \bL \varphi
\]
where \mbox{$p \in \at$} is a proposition.

We assume that ASP notation is transformed into its logical counterpart: $\Not F$ is represented as $\neg F$, commas in the body are replaced by conjunctions and rules $G \leftarrow F$ are written from left to right $F \to G$.

Intuitively, the new construct, $\bL\varphi$, is read as ``it is reasonable (for the planning agent) to believe~$\varphi$.''
Weak or intuitionistic negation is defined as usual: $\neg\varphi \eqdef \varphi \to \bot$.
The knowledge modality is defined as true belief:
$\bK \varphi \eqdef \varphi \wedge \bL \varphi$.
We also introduce the following abbreviations:
\begin{IEEEeqnarray*}{r?C?l}
\bU p &\eqdef& \neg p \wedge \neg\!\sneg p
\\
\bO p &\eqdef&  (p \to \bL \neg\neg p) \wedge (\sneg p \to \bL \neg\neg\!\sneg p) \wedge (\bU p \to \bL \neg\neg\!\bU p)
\end{IEEEeqnarray*}
whose respective intuitive meanings are that the value of proposition $p \in \at$ is \emph{undefined} and that the actual value of proposition $p \in \at$ is \emph{observed}.
Note that when an atom $p$ is observed, the agent's beliefs have to agree with the actual value of the atom $p$. The use of double negation here implies that only the agent's beliefs will be modified, without providing any justification for believing $p$.
Besides, we assume all previous abbreviations too, that is,
\mbox{$(\fF \leftarrow \fG) \eqdef (\fG \to \fF)$},
\
\mbox{$\fF \leftrightarrow \fG \eqdef (\fF \to \fG) \wedge (\fG \to \fF)$},
and
\ $\top \eqdef \bot\to\bot$.
An \emph{autoepistemic theory}~$\Gamma$ is a set of autoepistemic formulas as defined above.
When a theory is a singleton, we will usually write just $\varphi$ instead of~$\set{\varphi}$.

A \emph{literal}~$L$ is either a proposition $p \in \at$ or a proposition preceded by strong negation~$\sneg p$
and by $\lit \eqdef \at \cup \{\sneg p \mid p \in \at\}$ we denote the set of all literals over the signature~$\at$.

We define next an (autoepistemic) HT-model, as a combination of modal epistemic logic with the logic of here-and-there (HT)~\cite{heyting30a}, an intermediate logic with two intuitionistic worlds, $h$ (standing for ``here'') and $t$ (read as ``there'') satisfying $h \leq t$.
\begin{definition}[HT-Model]
Given a set of propositional symbols~$\at$,
an \emph{\htmodel} is a
quadruple
\mbox{$\cM = \tuple{W,\cK,V^h,V^t}$}
where
\begin{itemize}
\item $W$ is a set of worlds,
\item $\cK \subseteq W \times W$ is an accessibility relation on $W$, and
\item $V^x : W \longrightarrow 2^{\lit}$ is a valuation with $x \in \set{h,t}$ 
such that \mbox{$V^h(w) \subseteq V^t(w)$} 
for all $w \in W$.
\end{itemize}
$D(\cM) = W$
denotes the domain of $M$.
A \emph{belief \htmodel} is an \htmodel where $\cK = \cK \times \cK'$ with $\cK' = \cK \setminus \set{w_0}$ for some \emph{distinguish world} $w_0 \in \cK$.
\qed
\end{definition}
A \htmodel $\cM = \tuple{W,\cK,V^h,V^t}$ is called \emph{total}
iff $V^h = V^t$.
Furthermore, by $\cM^t \eqdef \tuple{W,\cK,V^t,V^t}$ we denote the total model corresponding to~$\cM$.
Satisfaction of autoepistemic formulas is then given by the following recursive definition:
\begin{itemize}
\item $\cM, w \not\models \bot$,
\item $\cM, w \models L$ iff $L \in V^h(w)$ for any $L \in \lit$,
\item $\cM, w \models \varphi_1 \wedge \varphi_2$ iff $\cM, w \models \varphi_1$ and $\cM, w \models \varphi_2$,
\item $\cM, w \models \varphi_1 \vee \varphi_2$ iff $\cM, w \models \varphi_1$ or $\cM, w \models \varphi_2$,

\item $\cM, w \models \varphi_1 \to \varphi_2$ iff
both $\cM, w \not\models \varphi_1$ or $\cM, w \models \varphi_2$
and\\
\hspace*{4.6cm}
$\cM^t, w \not\models \varphi_1$ or $\cM^t, w \models \varphi_2$,


\item $\cM, w \models \bL \varphi$ iff $\cM, w' \models \varphi$ for all $w'$ with $(w,w') \in \cK$
\end{itemize}
As usual, we say that $\cM$ is an \emph{\htmodel} of some theory~$\Gamma$, in symbols \mbox{$\cM \models \Gamma$}, iff
\mbox{$\cM,w\models \varphi$} for every world~$w \in D(\cM)$ and every formula~$\varphi \in \Gamma$.
As mentioned before, when $\Gamma = \set{\varphi}$ is a singleton we will omit the brackets, so that
\mbox{$\cM \models \varphi$} stands for \mbox{$\cM \models \set{\varphi}$}
and holds iff \mbox{$\cM,w\models \varphi$} for every world~$w \in D(\cM)$.

\begin{definition}[Bisimulation between HT-models]\label{def:int.prec}
Let
\mbox{$\cM_1 = \tuple{W_1,\cK_1,V_1^h,V_1^t}$}
and
\mbox{$\cM_2 = \tuple{W_2,\cK_2,V_2^h,V_2^t}$}
be two \htmodels.
Given some binary relation $Z \subseteq W_1\times W_2$,
we write $\cM_1 \preceq^Z \cM_2$ iff
\begin{itemize}
\item every $(w_1,w_2) \in Z$ satisfies $V^t(w_1) = V^t(w_2)$ and  $V^h(w_1) \subseteq V^h(w_2)$,

\item for every $(w_1,w_1') \in \cK_2$, there is $(w_2,w_2') \in \cK_2$ such that $(w_1',w_2') \in Z$,

\item for every $(w_2,w_2') \in \cK_2$, there is $(w_1,w_1') \in \cK_2$ such that
\mbox{$(w_1',w_2') \in Z$}.
\end{itemize}
We write $\cM_1 \preceq \cM_2$ iff there is a total relation $Z$ s.t. $\cM_1 \preceq^Z \cM_2$.
We also say that $\cM_1$ and $\cM_2$ are \emph{bisimilar}, in symbols $\cM_1 \approx \cM_2$, iff
there is a total relation $Z$ s.t $\cM_1 \preceq^Z \cM_2$ and $\cM_2 \preceq^Z \cM_1$.
As usual, we write
$\cM_1 \prec \cM_2$ iff $\cM_1 \preceq \cM_2$ and $\cM_1 \not\approx \cM_2$.\qed
\end{definition}

\begin{definition}[Equilibrium model]
A total belief \htmodel $\cM$ of some theory~$\Gamma$ is said to be an \emph{equilibrium model} of ~$\Gamma$ iff there is no other belief \htmodel~$\cM'$ of $\Gamma$ such that $\cM' \prec \cM$.\qed
\end{definition}

Given some information cell $\cM = \tuple{W,\cK,V}$ and some set of literals
\mbox{$I\subseteq \lit$},
by $\cM + I$ we denote the total belief \htmodel
$\cM' = \tuple{W',\cK',V',V'}$
where
\mbox{$W' = \set{w_0} \cup W$} with $w_0 \notin W$, \, $\cK' = W' \times W$ and
$V' : W' \longrightarrow 2^{\lit}$
satisfies $V'(w) = V(w)$ for all $w \in W$ and $V'(w_0) = I$.

\begin{definition}[World view]
Given a set of propositions~$\at$,
an information cell $\cM = \tuple{W,\cK,V}$ over $\lit$ is called a \emph{world view} of some theory~$\Gamma$ iff the following two conditions hold:
\begin{itemize}
\item $\cM + V(w)$ is an equilibrium model of $\Gamma$, for every world $w \in W$,

\item $\cM + I$ is not an equilibrium model of $\Gamma$ for every set of literals $I \subseteq \lit$ satisfying $I \neq V(w)$ for all $w \in W$, and

\item either $p \notin V(w)$ or $\sneg p \notin V(w)$ for all $p \in \at$ and $w \in W$.
\end{itemize}
We say that a theory~$\Gamma$ is \emph{consistent} iff it has some world view and by $\WV[\Gamma]$ we denote the set of all world views of~$\Gamma$.\qed
\end{definition}

\begin{examplecont}{ex:pink}\label{ex:pink.formula}
For instance, the formula
\begin{gather}
\varphi_0 = \sneg v \wedge \sneg l \wedge (r \vee \sneg r) \wedge \sneg s \wedge \sneg d
	\label{eq:formula.state0}
\end{gather}
has a unique world view that is depicted in Figure~\ref{fig:pink.state0.auto}.
\darkred
Note that every propositional theory has a unique world view~\cite{Cabalar2019faeel} that corresponds to the set of all answer sets of the theory.
Furthermore, since $\varphi_0$ contains no negation, its answer sets coincide with its minimal classical models when we treat each strong negated literal $\sneg p$ as a new atom.\qed
\end{examplecont}

%% file: deasp.tex

\section{Dynamic Epistemic Logic with ASP Updates: DEL[ASP]}

In this section, we present the major contribution of this paper, DEL[ASP], an instance of the abstract DEL framework where updating objects correspond to logic programs.
Our motivation is twofold: on the one hand, to allow unrestricted use of indirect effects (both ontic and epistemic); on the other hand, to preserve the ASP representation of non-epistemic planning problems without need of any adjustment or modification.
We illustrate these two objectives through our running example.

\subsection{Characterising information cells in FAEEL}

Let us start by showing how any information cell can be represented by some autoepistemic formula in FAEEL.
Note that world views are an information cell over $\lit$, so they represent a kind of three valued epistemic models where each proposition $p$ can be true $p \in V(w)$, false $\sneg p \in V(w)$ or undefined $p,\sneg p \notin V(w)$.
We will show here how (two-valued) information cells over $\at$ can be simply represented as propositional formulas in \FAEEL, allowing to map these three valued epistemic models into standard two valued ones.

\begin{examplecont}{ex:pink.formula}\label{ex:pink.formula.b}
Continuing with our running example,
we can see now that this model satisfies either \mbox{$p \in V(w)$} or \mbox{$\sneg p \in V(w)$} for every proposition \mbox{$p \in \set{v,l,r,s,d}$} and world $w \in \set{w_1,w_2}$.
Hence, we can map this model into a (two-valued) information cell by considering as true every proposition $p \in V(w)$ and as false every proposition satisfying \mbox{$\sneg p \in V(w)$}.
It is easy to see that the obtained information cell is precisely the model $\cM_0$
depicted in Figure~\ref{fig:pink.state0-4}a,
that is,
the epistemic model corresponding to the initial situation of Example~\ref{ex:pink}.
In this sense, we can use the formula~$\varphi_0$ to represent the initial state of this example.\qed
\end{examplecont}

\begin{figure}
\begin{center}
\begin{minipage}{\textwidth}
\begin{center}
\begin{tikzpicture}[align=center,node distance=6cm, framed, background rectangle/.style={draw=black, rounded corners}]]
\node[point] (w) [label=below:$w_1 : \set{\sneg v, \sneg l , r , \sneg s ,\sneg d }$, label={[label distance=2mm]175:$\cM_0^{wv}$} ] {};
\node[point] (v) [label=below:$w_2 : \set{\sneg v, \sneg l , \sneg r , \sneg s ,\sneg d }$, right=of w] {};
\path[-] (w) edge (v);
\end{tikzpicture}
\end{center}
\end{minipage}
\end{center}
	\caption{Unique world view of the formula~$\varphi_0 = \sneg v \wedge \sneg l \wedge (r \vee \sneg r) \wedge \sneg s \wedge \sneg d$.}\label{fig:pink.state0.auto}
\end{figure}
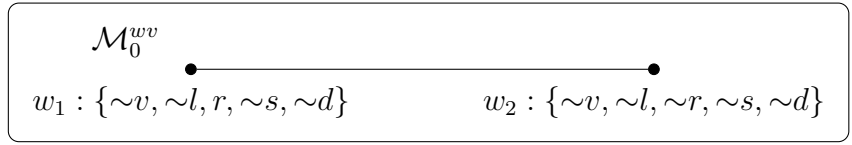

\begin{definition}\label{def:char.autoepistemic}
Given some information cell $\cM = \tuple{W,\cK,V}$,
its \emph{characteristic (autoepistemic) formula} is 
$\varphi_\cM \eqdef \bigvee_{w \in W}  \varphi_\cM^w$
where $\varphi_\cM^w$ is defined as follows:
\begin{IEEEeqnarray*}{c+x*}
\varphi_\cM^w \quad \eqdef \quad \Big(\bigwedge_{p \in V(w)} p \Big) \wedge \Big( \bigwedge_{p \in \at \setminus V(w)} \hspace{-5pt} \sneg p\Big)
&\qed
\end{IEEEeqnarray*}
\end{definition}

\begin{definition}[Bisimulation]
Given two models $\cM_1 = \tuple{W_1,\cK_1,V_1}$ and $\cM_2 = \tuple{W_2,\cK_2,V_2}$,
we say that they are \emph{bisimilar}, in symbols $\cM_1 \approx \cM_2$,
if and only if $\tuple{W_1,\cK_1,V_1,V_1} \approx \tuple{W_2,\cK_2,V_2,V_2}$.\qed
\end{definition}

\begin{definition}
Given a set of propositions~\mbox{$P \subseteq \at$},
we say that an \htmodel \mbox{$\cM =\tuple{W,\cK,V^h,V^t}$} over $\lit$ is \mbox{\emph{$P$-classical}} iff every world $w \in W$ and proposition $p \in P$ satisfy
that either $p \in V^h(w)$ or $\sneg p \in V^h(w)$ holds.
A theory~$\Gamma$ is \mbox{\emph{$P$-classical}} iff it is consistent
and, in addition, every world view is \mbox{$P$-classical}.
\qed
\end{definition}

\begin{definition}
Given a set of propositions~\mbox{$P \subseteq \at$}
and any \mbox{$P$-classical} total \htmodel \mbox{$\cM =\tuple{W,\cK,V^t,V^t}$} over~$\lit$, by \mbox{$\cM \downarrow P = \tuple{W,\cK,V}$} we denote the model over~$P$
where
\mbox{$V : W \longrightarrow 2^{P}$}
is a valuation
satisfying
$V(w) = V^t(w) \cap P$ for every world~$w \in W$.\qed
\end{definition}

\begin{proposition}\label{lem:char.autoepistemic}
Let $\cM$ be an information cell over~$\at$.
Then, $\varphi_\cM$ has a unique world view $\cM'$
and we have that $\cM$ and $\cM' \downarrow \at$ are bisimilar.
\end{proposition}

\begin{proof}
First note that, since $\varphi_\cM$ is a propositional formula, it has a unique world view~$\cM'$~\cite[Proposition~3]{Cabalar2019faeel}.
Let $\cM = \tuple{W,\cK,V}$ and $\cM' = \tuple{W',\cK',V'}$.
Then, we have that $w' \in W'$ iff $V'(w')$ is a stable model of $\varphi_\cM$.
Note also that the stable models of $\varphi_\cM$ are exactly its classical models understood as sets of literals.
Hence, for every $w' \in W'$, there is some $w \in W$ such that $V'(w') = V(w) \cup \sneg\,(\at \setminus V(w))$ and vice-versa.
Consequently, $\cM$ and $\cM' \downarrow \at$ are bisimilar.
\end{proof}

\begin{examplecont}{ex:pink.formula.b}\label{ex:pink.formula.c}
Continuing with our running example, we have $\varphi_{\cM_0} = \varphi_{\cM_0}^{w_1} \vee \varphi_{\cM_0}^{w_2}$ with
\begin{IEEEeqnarray*}{l C l}
\varphi_{\cM_0}^{w_1} &=& \sneg v \wedge \sneg l \wedge \phantom{\sneg} r \wedge \sneg s \wedge \sneg d
\\
\varphi_{\cM_0}^{w_2} &=& \sneg v \wedge \sneg l \wedge \sneg r \wedge \sneg s \wedge \sneg d
\end{IEEEeqnarray*}
By applying distributivity of conjunctions over disjunctions, it is easy to see that $\varphi_{\cM_0}$ is classically (and intuitionistically) equivalent to~\eqref{eq:formula.state0}.
As a result, $\cM_0^{wv}$ is the unique world view of $\varphi_{\cM_0}$
and, as expected from Proposition~\ref{lem:char.autoepistemic}, it can be checked that it satisfies $\cM_0^{wv} \downarrow \at = \cM_0$.\qed
\end{examplecont}

\subsection{Epistemic Model Updates with FAEEL}

In this section, we show how autoepistemic equilibrium logic can be used to define epistemic model updates just by using an extended signature.
Given a set of propositions $S \subseteq \at$,
we define
$\previous S \eqdef \setm{ \previous p }{ p \in S \cap \at}$
and
$\atbi = \at \cup \previous \at$
where $\previous p$ intuitively means that $p$ is true in the previous state.
It will also be convenient to use $\previous$ in front of any propositional formula~$\varphi$ such that $\previous \varphi$ is as an abbreviation for the formula obtained by writing $\previous$ in front of every proposition occurring in $\varphi$.
\begin{figure}
\begin{center}
\begin{minipage}{\textwidth}
\begin{center}
\begin{tikzpicture}[align=center,node distance=7.5cm, framed, background rectangle/.style={draw=black, rounded corners}]]
\node[point] (w) [label=below:$w_1 : \previous\set{\sneg v, \sneg l , r , \sneg s , \sneg d }$ \\ $\cup\, \set{ \move, v, \sneg l , r , \sneg s ,\sneg d }$, label={[label distance=2mm]175:$\cM_1^{wv}$} ] {};
\node[point] (v) [label=below:$w_2 : \previous\set{\sneg v, \sneg l , \sneg r , \sneg s ,\sneg d }$ \\ $\cup\, \set{ \move, v, \sneg l , \sneg r , \sneg s ,\sneg d }$, right=of w] {};
\path[-] (w) edge (v);
\end{tikzpicture}
\end{center}
\end{minipage}
\end{center}
	\caption{Unique world view of the program~\theory\ref{prg:pink.trans0-1}.}
	\label{fig:pink.trans0-1}
\end{figure}
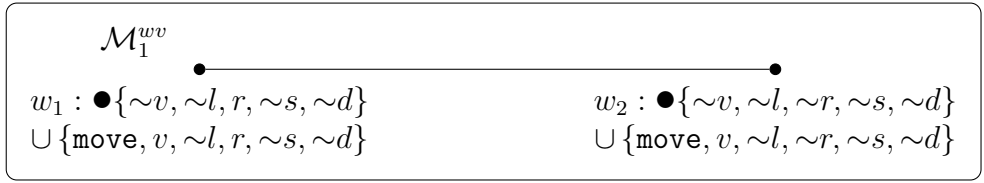

\begin{examplecont}{ex:pink.formula}\label{ex:pink.trans0-1}
Let $\GammaPink$ be a theory containing formulas~\eqref{eq:takel.postcondition}-\eqref{eq:obs} and let $\newtheory\label{prg:pink.trans0-1} = \GammaPink \cup \set{ \move , \previous \varphi_{\cM_0} }$.
This program has a unique world view shown in Figure~\ref{fig:pink.trans0-1}.
Note that, if we disregard all the information corresponding to the previous situation (that is all literals preceded by~$\previous$) and the action~$\move$, then we have the same information as the epistemic model~$\cM_1$ in Figure~\ref{fig:pink.state0-4}a.
In other words, $\theory\ref{prg:pink.trans0-1}$ encodes the transition that occurs between the epistemic models~$\cM_0$ and~$\cM_1$ when executing action~$\move$.\qed
\end{examplecont}


As shown in the example above, we can represent the transition between two epistemic models as an autoepistemic theory.
Let us now formalise this intuition.
We begin introducing some auxiliary definitions.

Given a set of epistemic models $\cS = \set{ \cM_1, \cM_2, \dotsc }$ where each $\cM_i$ is a model over a set of atoms $\at$ of the form $\cM_i = \tuple{W_i,\cK_i,V_i}$ and
satisfying $W_i \cap W_j = \emptyset$ for all $\cM_i,\cM_j \in \cS$ with $i \neq j$, by
\mbox{$\bigsqcup \cS \eqdef \tuple{W,\cK,V}$},
we denote an epistemic model where
\begin{itemize}
\item $W' = \bigcup\setm{ W_i }{ \cM_i \in \cS }$,
\item $\cK' = \bigcup\setm{ \cK_i }{ \cM_i \in \cS }$, and
\item $V' : W' \longrightarrow 2^{\at}$ with $V'(w) = V_i(w)$ for all $w \in W_i$ and all $\cM_i \in \cS$.
\end{itemize}
As usual, if $\cS = \set{\cM_1 , \cM_2 }$,
we just write $\cM_1 \sqcup \cM_2$ instead of $\bigsqcup \set{\cM_1,\cM_2}$.

\begin{definition}
Let $P$ be a set of atoms and $\Gamma$ be some \mbox{$P$-classical}  autoepistemic theory.
Then, by
\mbox{$\model{\Gamma, P} \eqdef \bigsqcup \setm{ \cM \downarrow P }{ \cM \in \WV[\Gamma ] }$}
we denote the epistemic model capturing all the world views of $\Gamma$ projected into the vocabulary~$P$.
If $\Gamma$ is not \mbox{$P$-classical}, we assume that
\mbox{$\model{\Gamma, P}$} is not defined.\qed
\end{definition}

\darkred
In other words, for every $P$-classical autoepistemic theory~$\Gamma$,
$\model{\Gamma, P}$ is the two-valued epistemic model that has an information cell for every world view of~$\Gamma$ such that the valuation of every proposition $p \in P$ in every world in $\model{\Gamma, P}$ corresponds to the valuation of that proposition in that same world in the corresponding world view.
\black
Recall that, in a $P$-classical theory, all its worlds views satisfy either $p$ or $\sneg p$ for every proposition~$p \in P$.
This is necessary so it is possible to map three-valued world views into two-valued epistemic models.
We could remove this restriction by allowing three-valued epistemic models, but we have decided to stay as close as possible to the original DEL semantics, which is only two-valued.


\begin{examplecont}{ex:pink.trans0-1}\label{ex:pink.trans0-1b}
Note that if $P \subseteq \at$, then the epistemic model $\model{\Gamma,P}$ always corresponds to the current situation, discarding all information about the previous one.
In this sense, 
if we consider the program $\theory\ref{prg:pink.trans0-1}$ of Example~\ref{ex:pink.trans0-1} and the epistemic model~$\cM_1$ of Figure~\ref{fig:pink.state0-4}b,
we have that $\model{\theory\ref{prg:pink.trans0-1},P} = \cM_1$
where $P = \set{v,d,r,l,s}$ is the set of fluent of Example~\ref{ex:pink}.

As a further example, consider now the theory~$\newtheory\label{prg:pink.trans1-2} = \GammaPink \cup \set{\flick,\previous\varphi_{\cM_1}}$.
Then,
$\theory\ref{prg:pink.trans1-2}$
has two world views which correspond to the two cell informations in the epistemic model~$\cM_2$ depicted in Figure~\ref{fig:pink.state0-4}c.
This explains why we need to join together all the world views of the theory in a single epistemic model.
Every world view, which becomes an information cell, represents the knowledge the agent will have after executing the action $\flick$, while the set of all world views represent the knowledge the agent had before executing it.\qed
\end{examplecont}

Let us now define the transition between states borrowing the notion of product update from DEL.

\begin{definition}
Given an information cell~$\cM$ over~$P \subseteq \at$ and a theory~$\Gamma$ over~$\atbi$
such that $\Gamma \cup \set{ \previous\varphi_{\cM} }$ is $\atbi$-classical,
we define:
\begin{itemize}
\item the \emph{product update} of $\cM$ with respect to $\Gamma$
as the epistemic model \mbox{$\cM \otimes \Gamma \,\eqdef\, \model{\Gamma \cup \set{ \previous\varphi_{\cM} },P}$}.

\item
the binary relation $\erelation{\cM}{\Gamma} \subseteq D(\cM) \times D(\cM \otimes \Gamma)$
s.t.
\mbox{$(w,w') \in \erelation{\cM}{\Gamma}$}
iff
$\cM',w' \models \previous\varphi_\cM^w$
with $\cM' = \model{\Gamma \cup \set{ \previous\varphi_{\cM} }, \previous\at}$.\qed
\end{itemize}

\end{definition}

\begin{examplecont}{ex:pink.trans0-1b}\label{ex:pink.trans0-1c}
It is now easy to see that
\begin{IEEEeqnarray*}{cCcCcClCl}
\cM_0 \otimes \theory\ref{prg:pink.trans-move} 
	&=& \model{\theory\ref{prg:pink.trans-move}  \cup \set{ \previous\varphi_{\cM_0} },P} 
	&=& \model{\theory\ref{prg:pink.trans0-1},P} 
	&=& \cM_1^{wv} \downarrow P
	&=& \cM_1
\\
\cM_1 \otimes \theory\ref{prg:pink.trans-flick} 
	&=& \model{\theory\ref{prg:pink.trans-flick}  \cup \set{ \previous\varphi_{\cM_0} },P} 
	&=& \model{\theory\ref{prg:pink.trans1-2},P} 
	&=& \cM_2
\end{IEEEeqnarray*}
with $\newtheory\label{prg:pink.trans-move} =  \GammaPink \cup \set{ \move }$
and
$\newtheory\label{prg:pink.trans-flick} =  \GammaPink \cup \set{ \move }$
respectively being the theories representing the execution of the action $\move$ and $\flick$ according to program~$\GammaPink$, $\cM_1^{wv}$ the epistemic model depicted in \mbox{Figure~\ref{fig:pink.trans0-1}}
and $\cM_1$ and $\cM_2$ the epistemic models depicted in \mbox{Figure~\ref{fig:pink.state0-4}b}
and
\mbox{Figure~\ref{fig:pink.state0-4}c}.
In other words, $\cM_1$ is the result of executing action $\move$ in epistemic model $\cM_0$ according to the description provided by $\GammaPink$.
Furthermore, for each world \mbox{$w \in \set{w_1,w_2}$}, we have 
\mbox{$\cM_1^{wv}\!, w \models \previous\varphi_\cM^w$}.
In its turn, this implies \mbox{$\cM_1^{wv}\!\downarrow \previous \at, w \models \previous\varphi_\cM^w$}
and, thus, that
$\erelation{\cM_0}{\theory\ref{prg:pink.trans-move}} = \set{(w_1,w_1),(w_2,w_2)}$ is the identity\footnote{Note that, for the sake of clarity, the names of worlds have been chosen so that
$\erelation{\cM_0}{\theory\ref{prg:pink.trans-move}}$ is the identity, but this is not necessarily the case. In fact, worlds from $\cM_0$ and $\cM_1$ could be disjoint or even be switched so that $w_1$ could be called $w_2$ and vice-versa.}, that is, it maps each world in $\cM_0$ to a world in~$\cM_1$ with the same name.
Similarly, $\cM_2$ is the result of executing the action $\flick$ in the epistemic model $\cM_1$ according to the description provided by $\GammaPink$
and we can check that
$\erelation{\cM_1}{\theory\ref{prg:pink.trans-flick}}$
is also the identity.
\qed
\end{examplecont}

We define now the updating evaluation for ASP epistemic specifications for epistemic models.
In a nutshell, this evaluation is the result of combining the evaluation for each individual information cell in the model.

\begin{definition}[ASP updating evaluation]
Given any epistemic model~$\cM$ and theory~$\Gamma$,
the \emph{ASP updating evaluation}
is a pair $\tuple{\otimes,\cR}$
satisfying
\begin{IEEEeqnarray*}{l ?C? l +x*}
\cM \otimes \Gamma &\eqdef& \bigsqcup\setm{\cM' \otimes \Gamma}{ \cM' \in \cell{\cM}}
\\
\erelation{\cM}{\Gamma} 
&\eqdef& \bigcup\setm{\cM' \otimes \Gamma}{ \cM' \in \cell{\cM}} &\qed
\end{IEEEeqnarray*}
\end{definition}

We can now directly apply Definition~\ref{def:del.sat}
to obtain the satisfaction of DEL[ASP] formulas, that is,
DEL formulas in which the updating objects are autoepistemic theories.


\begin{figure}
\begin{center}
\begin{minipage}{0.18\textwidth}
\begin{center}
\begin{tikzpicture}[align=center,node distance=2.5cm, framed, background rectangle/.style={draw=black, rounded corners}]]
\node[point] (w) [label=below:$w_1 : \underline{vl}r\underline{sd}$, label={[label distance=2mm]90:$\cM_0'$}] {};
\end{tikzpicture}
\\
(a)
\end{center}
\end{minipage}
\begin{minipage}{0.18\textwidth}
\begin{center}
\begin{tikzpicture}[align=center,node distance=2.5cm, framed, background rectangle/.style={draw=black, rounded corners}]]
\node[point] (w) [label=below:$w_1 : v\underline{l}r\underline{sd}$, label={[label distance=2mm]90:$\cM_0' \otimes \theory\ref{prg:pink.trans-move}$}] {};
\end{tikzpicture}
\\
(b)
\end{center}
\end{minipage}
\begin{minipage}{0.18\textwidth}
\begin{center}
\begin{tikzpicture}[align=center,node distance=2.5cm, framed, background rectangle/.style={draw=black, rounded corners}]]
\node[point] (w) [label=below:$w_1 : v\underline{l}r\underline{s}d$, label={[label distance=2mm]90:$\cM_0' \otimes \theory\ref{prg:pink.trans-move} \otimes \theory\ref{prg:pink.trans-taker}$}] {};
\end{tikzpicture}
\\
(c)
\end{center}
\end{minipage}
\begin{minipage}{0.42\textwidth}
\begin{center}
\begin{tikzpicture}[align=center,node distance=2.5cm, framed, background rectangle/.style={draw=black, rounded corners}]]
\node[point] (w) [label=below:$w_1 : \underline{vl}r\underline{s}d$, label={[label distance=2mm]90:$\cM_0' \otimes \theory\ref{prg:pink.trans-move} \otimes \theory\ref{prg:pink.trans-taker} \otimes \theory\ref{prg:pink.trans-move}$}] {};
\end{tikzpicture}
\\
(d)
\end{center}
\end{minipage}
\end{center}
	\caption{Epistemic models corresponding execution of the sequence of actions $\tuple{\move,\taker,\move}$ in initial state of Example~\ref{ex:pink.simple}, represented here by the model $\cM_0'$ in (a).
	}\label{fig:pink.simply.states}
\end{figure}

\begin{examplecont}{ex:pink.simple}
Let us now resume the simplified version of our running example introduced in Example~\ref{ex:pink.simple}.
In this case, the initial situation can be represented by an epistemic model $\cM_0'$ depicted in Figure~\ref{fig:pink.simply.states}a.
Then, it can be checked that
$$\cM_0' \models \bK[\theory\ref{prg:pink.trans-move}][\theory\ref{prg:pink.trans-taker}][\theory\ref{prg:pink.trans-move}] (\sneg v \wedge d)$$
holds
with
$\newtheory\label{prg:pink.trans-taker} =  \GammaPink \cup \set{ \taker }$.
In other words, the thief knows that after executing the sequence of actions $\tuple{\move,\taker,\move}$ 
she will be out of the vault with the diamond.
That is, this sequence of actions is a valid plan that achieves the goal of getting out of the vault with the diamond regardless of the actual initial situation.
This means that this is a conformant plan.

For the sake of completeness, Figures~\ref{fig:pink.simply.states}b, c and d respectively depict the epistemic models
$\cM_0' \otimes \theory\ref{prg:pink.trans-move}$,
$\cM_0' \otimes \theory\ref{prg:pink.trans-move} \otimes \theory\ref{prg:pink.trans-taker}$
and
$\cM_0' \otimes \theory\ref{prg:pink.trans-move} \otimes \theory\ref{prg:pink.trans-taker} \otimes \theory\ref{prg:pink.trans-move} $.\qed
\end{examplecont}

\begin{figure}
\begin{center}
\begin{minipage}{0.4\textwidth}
\begin{center}
\begin{tikzpicture}[align=center,node distance=2.5cm, framed, background rectangle/.style={draw=black, rounded corners}]]
\node[point] (w) [label=below:$w_1 : \underline{vl}r\underline{sd}$,
label={[label distance=2mm]90:$\cM_1 \otimes \theory\ref{prg:pink.trans-ttakel} $} ] {};
\node[point] (v) [label=below:$w_2 : \underline{vlrs}d$, right=of w] {};
\path[-] (w) edge (v);
\end{tikzpicture}
\\
(a)
\end{center}
\end{minipage}
\hspace{1cm}
\begin{minipage}{0.4\textwidth}
\begin{center}
\begin{tikzpicture}[align=center,node distance=2.5cm, framed, background rectangle/.style={draw=black, rounded corners}]]
\node[point] (w) [label=below:$w_1 : v\underline{l}r\underline{s}d$, label={[label distance=2mm]90:$\cM_1 \otimes \theory\ref{prg:pink.trans-ttakel} \otimes \theory\ref{prg:pink.trans-ttaker}$} ] {};
\node[point] (v) [label=below:$w_2 : v\underline{lrs}d$, right=of w] {};
\path[-] (w) edge (v);
\end{tikzpicture}
\\
(b)
\end{center}
\end{minipage}
\end{center}
	\caption{Epistemic models corresponding to (a) the execution of $\ttakel$ in the model~$\cM_1 = \cM_0 \otimes \theory\ref{prg:pink.trans-ttakel}$ and (b) the execution of $\ttaker$ in the resulting state.
	 }\label{fig:pink.conformant.states}
\end{figure}

\begin{examplecont}{ex:pink}\label{ex:pink.try}
As a further example,
consider another variation of Example~\ref{ex:pink} where we have actions $\ttakel$ and $\ttaker$ that are similar to $\takel$ and $\ttakel$, but that can be executed even when the diamond is not in the right location, having no effect in such case.
This can be represented by a theory $\GammaPink'$ obtained from $\GammaPink$ by replacing rules~(\ref{eq:takel.postcondition}-\ref{eq:taker.pre2}) by the following rules:
\begin{IEEEeqnarray}{c C l}
d &\leftarrow& \ttakel \wedge \previous\sneg r
	\label{eq:ttakel}
\\
\bot &\leftarrow& \ttakel \wedge \sneg \previous v
	\label{eq:ttakel.pre}
\\
d &\leftarrow& \ttaker \wedge  \previous r
	\label{eq:ttaker}
\\
\bot &\leftarrow& \ttaker \wedge \sneg \previous v
	\label{eq:ttaker.pre}
\end{IEEEeqnarray}
Now, we can check that
\begin{gather}
\cM_0 \models \bK[\theory\ref{prg:pink.trans-move}][\theory\ref{prg:pink.trans-ttakel}][\theory\ref{prg:pink.trans-ttaker}][\theory\ref{prg:pink.trans-move}] (\sneg v \wedge d)
    \label{eq:1:ex:pink.try}
\end{gather}
holds
with
$\newtheory\label{prg:pink.trans-ttakel} =  \GammaPink' \cup \set{ \ttakel }$
and
$\newtheory\label{prg:pink.trans-ttaker} =  \GammaPink' \cup \set{ \ttaker }$.
Recall that
$\cM_0 \otimes \theory\ref{prg:pink.trans-move} = \cM_1$
is the epistemic model depicted in Figure~\ref{fig:pink.state0-4}a.
We also can check that
$\cM_0 \otimes \theory\ref{prg:pink.trans-move} \otimes \theory\ref{prg:pink.trans-ttakel} \otimes \theory\ref{prg:pink.trans-ttaker} \otimes \theory\ref{prg:pink.trans-move} = \cM_4$ is the epistemic model depicted in Figure~\ref{fig:pink.state0-4b}b.
Figure~\ref{fig:pink.conformant.states} depicts the epistemic models corresponding to intermediate states.
With these models, we can see that~\eqref{eq:1:ex:pink.try}
holds.
That is, the thief knows that after executing the sequence of actions actions
$\tuple{\move,\ttakel,\ttaker,\move}$ she will be outside the vault with the diamond.
Therefore, this sequence of actions constitutes a conformant plan for this problem.
Note that the thief achieves her goal without ever getting to know where the diamond actually was.\qed
\end{examplecont}

\section{Conditional Planning in DEL[ASP]}

In this section, we show how to use DEL[ASP] to represent conditional plans.
Let start by defining what a plan is by
introducing the following plan language from~\cite{AndersenBJ12}.

\begin{definition}[Plan Language]
Given disjunct sets of actions $\cA$ and fluents~$\cF$, a \emph{plan} is an expression $\pi$ built with the following grammar:
\begin{gather*}
\pi \quad::=\quad a \mid \sskip \mid \iif \bK\varphi \tthen \pi \eelse \pi \mid \pi;\pi
\end{gather*}
where
\mbox{$a \in \cA$}
and
$\varphi$ is a formula over~$\cF$.
We write {\rm ($\iif \bK\varphi \tthen \pi$)} as a shorthand for the plan {\rm ($\iif \bK\varphi \tthen \pi \eelse \sskip$)}.\qed
\end{definition}

As mentioned in the introduction, conditional plans contain ``if-then-else'' structures that allow the agent to apply different strategics depending on the knowledge she has obtained along the execution of the plan.
For instance,
\begin{gather}
\move\;;\, \flick\;;\,\iif \bK r \tthen \taker \eelse \takel \;;\, \move 
	\label{eq:conditional.plan}
\end{gather}
is a conditional plan for the problem laid out in Example~\ref{ex:pink}.  It is a plan since, as we will prove next, the thief eventually takes the diamond out in all possible outcomes, and it is conditional because the third step contains an alternative decision. If the thief acts according to her knowledge about the diamond position at that point, the plan is guaranteed to succeed.
We will show that in fact, after executing the actions $\move$ and $\flick$, the thief knows that she will know where the diamond is.

Let us now formalise these intuitive ideas by providing a translation from plans into DEL[ASP] as follows:

\begin{definition}[Translation]
Let $\cA \subseteq \at$ and $\cF \subseteq \at$ be a pair of disjoint sets of propositions, respectively corresponding to actions and fluents.
The translation of a plan $\pi$ over $\cA$ applied to a formula $\psi$ over $\cF$, with respect to a theory $\Gamma$ over $\atbi$ is denoted as $\den{\pi}^\Gamma \psi$ and is recursively defined as:
\begin{IEEEeqnarray*}{l?C?l+x*}
\den{a}^\Gamma \, \psi &\eqdef& \tuple{\Gamma \cup \set{a}}\top \wedge [\Gamma \cup \set{a}]\psi
\\
\den{\sskip}^\Gamma \, \psi &\eqdef& \psi
\\
\den{\pi;\pi'}^\Gamma \, \psi &\eqdef& \den{\pi}^\Gamma\,(\den{\pi'}^\Gamma\,\psi)
\\
\den{\iif \bK\varphi \tthen \pi \eelse \pi' }^\Gamma \, \psi &\eqdef& (\bK\varphi \to \den{\pi}^\Gamma\psi) \wedge (\sneg\bK\varphi \to \den{\pi'}^\Gamma \psi)
\end{IEEEeqnarray*}
where $\varphi$ is any formula over $\cF$.\qed
\end{definition}

As a first remark, note that the translation $\den{a}^\Gamma$ of an action $a$ is always made by adding a constant theory $\Gamma$ that defines the behaviour of the action domain (fixing the transition relation). 
As a result, each elementary action in the plan becomes a complete autoepistemic theory $\Gamma \cup \{a\}$ in the translation.
When $\Gamma$ is clear from the context, we will simply write $\den{\pi}$ instead of $\den{\pi}^\Gamma$.
Conjunct $[\Gamma \cup \set{a}] \ \psi$ requires that $\psi$ becomes true in any resulting state whereas $\tuple{\Gamma \cup \set{a}}\top$ ensures that action $a$ is executable indeed.

We check next that the evaluation of plan~\eqref{eq:conditional.plan} corresponds to what we have already seen in Example~\ref{ex:pink.trans0-1c}.
For the sake of clarity, we gather together all rules of the theory~$\Gamma_{\pink}$ in Figure~\ref{fig:Gpink}.
\begin{figure}
\begin{center}
\begin{tabular}{ | l | l | }
\hline
\begin{minipage}{0.45\textwidth}
\begin{IEEEeqnarray*}{c C l +x*}
d &\leftarrow& \takel
	&\eqref{eq:takel.postcondition}
\\
\bot &\leftarrow& \takel, \sneg\previous v
	&\eqref{eq:takel.pre1}
\\
\bot &\leftarrow& \takel, \previous r
	&\eqref{eq:takel.pre2}
\\
d &\leftarrow& \taker
	&\eqref{eq:taker.postcondition}
\\
\bot &\leftarrow& \taker, \sneg \previous v
	&\eqref{eq:taker.pre1}
\\
\bot &\leftarrow& \taker, \sneg \previous r
	&\eqref{eq:taker.pre2}
\\
l &\leftarrow& \flick
	&\eqref{eq:flick.post}
\\
\bot &\leftarrow& \flick, \sneg \previous v
	&\eqref{eq:flick.pre}
\\
v &\leftarrow& \move, \sneg\previous v
	&\eqref{eq:move.post1}
\\
\sneg v &\leftarrow& \move, \previous v
	&\eqref{eq:move.post2}
\\
\end{IEEEeqnarray*}
\end{minipage}
&
\begin{minipage}{0.45\textwidth}
\begin{IEEEeqnarray*}{c C l +X+}
s &\leftarrow& v, l
	&\eqref{eq:see}
\\
\bO r &\leftarrow& v, l
	&\eqref{eq:obs}
\\
\sneg s &\leftarrow& \Not s
	&\eqref{eq:pink.8}
\\
v &\leftarrow& \previous v, \Not\!\sneg v
\\
\sneg v &\leftarrow& \sneg\previous v, \Not v
\\
l &\leftarrow& \previous l, \Not\!\sneg l
\\
\sneg l &\leftarrow& \sneg\previous l, \Not l
\\
r &\leftarrow& \previous r, \Not\!\sneg r
\\
\sneg r &\leftarrow& \sneg\previous r, \Not r
\\
d &\leftarrow& \previous d, \Not\!\sneg d
\\
\sneg d &\leftarrow& \sneg\previous d, \Not d
\end{IEEEeqnarray*}
\end{minipage}
\\
\hline
\end{tabular}
\end{center}
	\caption{Theory~$\Gamma_{\pink}$: the left column contains the direct effects and preconditions of actions while the right one contains the indirect effects and the inertia axioms.}
	\label{fig:Gpink}
\end{figure}

\begin{examplecont}{ex:pink.trans0-1c}\label{ex:pink.DEASP0-1b}
\normalfont
Going on with our running example, let us consider plans $\den{\pi}^{\GammaPink}$ simply denoted as $\den{\pi}$. We have seen that
$\cM_0 \otimes \theory\ref{prg:pink.trans-move} = \cM_1$
where
$\theory\ref{prg:pink.trans-move} =  \GammaPink \cup \set{ \move }$ was the theory representing the execution of  action $\move$ according to~$\GammaPink$.
We have also seen that $\erelation{\cM_0}{\theory\ref{prg:pink.trans-move}} = \set{(w_1,w_1),(w_2,w_2)}$ is the identity.
Similarly,
given theory
$\newtheory\label{prg:pink.obs} = \GammaPink \cup \set{\flick}$ representing the execution of action $\flick$,
we have
$\cM_1 \otimes \theory\ref{prg:pink.obs} = \cM_2$
and
$\erelation{\cM_1}{\theory\ref{prg:pink.obs}} = \set{(w_1,w_1),(w_2,w_2)}$.
Figure~\ref{fig:pink.temporal-r} shows these three models together with the corresponding relations~$\cR_{\cM_0,\theory\ref{prg:pink.trans-move}}$ and~$\cR_{\cM_1,\theory\ref{prg:pink.obs}}$.
Looking at this figure, we observe that $\cM_2,w_1 \models v \wedge \bK r$ and, thus, also $\cM_2,w_1 \models v \wedge (\bK r \vee \bK \sneg r)$.
From this we can conclude that
$\cM_1,w_1  \models [\theory\ref{prg:pink.obs}](v \wedge (\bK r \vee \bK \sneg r))$.
Note that
$\cM_1 \models \tuple{\theory\ref{prg:pink.obs}} \top$
holds and, thus, it follows 
$$\cM_1,w_1  \models \den{ \flick }(v \wedge (\bK r \vee \bK \sneg r))$$
Now we can check that
$\cM_0,w_1  \models [\theory\ref{prg:pink.trans-move}]\den{ \flick }(v \wedge (\bK r \vee \bK \sneg r))$
and
$\cM_0 \models \tuple{\theory\ref{prg:pink.trans-move}} \top$
hold
and, thus, we can conclude
$$\cM_0,w_1  \models \den{ \move } \big(\den{ \flick }(v \wedge (\bK r \vee \bK \sneg r)) \big)$$
An analogous reasoning, allow us to see that the same holds for
$\cM_0,w_2$ and, thus, we obtain
$$\cM_0   \models \bK \den{ \move } \big(\den{ \flick }(v \wedge (\bK r \vee \bK \sneg r)) \big)$$
By definition, these two facts imply
\begin{gather}
	\cM_0  \models \bK \den{ \move ; \flick }(v \wedge (\bK r \vee \bK \sneg r))
	\label{eq:1:ex:pink.DEASP0-1b}
\end{gather}
In other words, the thief knows that, after executing actions $\move$ and $\flick$, she will be inside the vault and that she will know where the diamond is.
So she will be ready for the next step: using her knowledge to decide what is the suitable action to continue the plan.
\qed
\end{examplecont}

\begin{figure}[htbp]
\begin{center}
\begin{tikzpicture}[align=center,node distance=1.5cm and 2.5cm ]
\node[point] (w) [label=above:$w_1 : \underline{vl}r\underline{sd}$] {};
\node[point] (v) [below=of w , label=below:$w_2 : \underline{vlrsd}$,  label={[label distance=10mm]-90:$\cM_0$} ] {};

\node[point] (w1) [right=of w, label=above:$w_1 : v\underline{l}r\underline{sd}$ ] {};
\node[point] (v1) [right=of v, label=below:$w_2 : v\underline{lrsd}$, label={[label distance=10mm]-90:$\cM_1$}] {};

\node[point] (w2) [right=of w1,label=below:$w_1 : vlrs\underline{d}$] {};
\node[point] (v2) [right=of v1, label=below:$w_2 : vl\underline{r}s\underline{d}$ , label={[label distance=10mm]-90:$\cM_2$}] {};

\node[point] (w3) [right=of w2, label=below:$w_1 : vlrsd$] {};
\node (v3) [right=of v2, label={[label distance=10mm]-90:$\cM_{31}$}] {};

\node[point] (w4) [right=of w3, label=below:$w_1 : \underline{v}lrsd$] {};
\node (v4) [right=of v3, label={[label distance=10mm]-90:$\cM_{41}$}] {};

\path[-] (w) edge (v);
\path[-] (w1) edge (v1);

\path[->,dashed] (w) edge (w1);
\path[->,dashed] (v) edge (v1);

\path[->,dashed] (w1) edge (w2);
\path[->,dashed] (v1) edge (v2);

\path[->,dashed] (w2) edge (w3);
\path[->,dashed] (w3) edge (w4);

\node[draw=black, rounded corners, fit={($(w)+(-30pt,30pt)$) ($(v)+(30pt,-60pt)$)}](FIt) {};
\node[draw=black, rounded corners, fit={($(w1)+(-30pt,30pt)$) ($(v1)+(30pt,-60pt)$)}](FIt1) {};
\node[draw=black, rounded corners, fit={($(w2)+(-30pt,30pt)$) ($(v2)+(30pt,-60pt)$)}](FIt2) {};
\node[draw=black, rounded corners, fit={($(w3)+(-30pt,30pt)$) ($(v3)+(30pt,-60pt)$)}](FIt3) {};
\node[draw=black, rounded corners, fit={($(w4)+(-30pt,30pt)$) ($(v4)+(30pt,-60pt)$)}](FIt4) {};

\node  at ($0.5*(v)+0.5*(v1)+(0,-80pt)$) (a1) {$\move$};
\node  at ($0.5*(v1)+0.5*(v2)+(0,-80pt)$) (a1) {$\flick$};
\node  at ($0.5*(v2)+0.5*(v3)+(0,-80pt)$) (a1) {$\taker$};
\node  at ($0.5*(v3)+0.5*(v4)+(0,-80pt)$) (a1) {$\move$};
\end{tikzpicture}
\vspace{-0.75cm}
\end{center}
	\caption{Execution of the sequence of actions $\tuple{\move,\flick,\taker,\move}$ starting at $\cM_0,w_1$ of Example~\ref{ex:pink}.
	We have $\cM_{i+1} = \cM_i \otimes (\Gamma_{\pink} \cup \set{a_i})$ with $a_i$ the corresponding action in the sequence.
	The dotted arrows depict the $\cR$ relation associated with the update of $\cM_i$ with respect to $\Gamma_{\pink} \cup \set{a_i}$.
	Note that action $\taker$ is not executable in $\cM_2,w_2$ and, as a result, $w_2$ has no associated world in $\cM_3$.
	 }\label{fig:pink.temporal-r}
\end{figure}

Let us now continue with the thief's reasoning process after the execution of the first two actions.

\begin{examplecont}{ex:pink.DEASP0-1b}\label{ex:pink.DEASP0-1b.plan}
\normalfont
We will show now that
\begin{gather}
\cM_2,w_1 \models \den{\iif \bK r \tthen \taker \eelse \takel ; \move } (v \wedge d)
	\label{eq:1:ex:pink.DEASP0-1b.plan}
\end{gather}
is satisfied.
First note that
\begin{IEEEeqnarray*}{l ?C? l}
\cM_{31} \otimes \theory\ref{prg:pink.trans-move} &=& \cM_{41}
\\
\erelation{\cM_2}{\theory\ref{prg:pink.trans-move}} &=& \set{(w_1,w_1)}
\end{IEEEeqnarray*}
and that $\cM_{41},w_1 \models \sneg v \wedge d$
and, thus, we get 
$\cM_{41},w_1 \models \den{ \move } (\sneg v \wedge d)$.
Let now
$\newtheory\label{prg:pink.trans2-3} \eqdef \GammaPink \cup \set{ \taker }$.
Then, we have
\begin{IEEEeqnarray*}{l ?C? l}
\cM_2 \otimes \theory\ref{prg:pink.trans2-3} &=& \cM_{31}
\\
\erelation{\cM_2}{\theory\ref{prg:pink.trans2-3}} &=& \set{(w_1,w_1)}
\end{IEEEeqnarray*}
from where we get $\cM_2,w_1 \models [\theory\ref{prg:pink.trans2-3}] \den{ \move } (\sneg v \wedge d)$.
Furthermore, this implies
$\cM_2,w_1 \models \den{\taker} \den{ \move } (\sneg v \wedge d)$,
which in its turn implies 
$$\cM_2,w_1 \models \bK r \to \den{\taker} \den{ \move } (\sneg v \wedge d)$$
Note that
$\cM_2,w_1 \models \bK r$
and, thus,
$$\cM_2,w_1 \models \sneg\bK r \to \den{\taker} \den{ \move } (\sneg v \wedge d)$$
also follows.
As a result, we can see that~\eqref{eq:1:ex:pink.DEASP0-1b.plan} holds.
Now we follow the reasoning from Example~\ref{ex:pink.DEASP0-1b} to show that
$\cM_0, w_1 \models \den{\eqref{eq:conditional.plan}} ( \sneg v \wedge d)$.
That is,~\eqref{eq:conditional.plan} is a plan that achieves the goal of Example~\ref{ex:pink} in the case that the diamond is in the right pedestal.
Analogously, Figure~\ref{fig:pink.temporal-l}
shows the models needed to prove $\cM_0, w_2 \models \den{\eqref{eq:conditional.plan}} (\sneg v \wedge d)$, that is, when the diamond was on the left.
As a result, we obtain
$\cM_0 \models \bK \den{\eqref{eq:conditional.plan}} (\sneg v \wedge d)$.
In other words, the thief knows that after executing~\eqref{eq:conditional.plan},
she will succeed in her goal: being outside of the vault with the diamond.\qed
\end{examplecont}
\begin{figure}[htbp]
\begin{center}
\begin{tikzpicture}[align=center,node distance=1.5cm and 2.5cm ]
\node[point] (w) [label=above:$w_1 : \underline{vl}r\underline{sd}$] {};
\node[point] (v) [below=of w , label=below:$w_2 : \underline{vlrsd}$,  label={[label distance=10mm]-90:$\cM_0$} ] {};

\node[point] (w1) [right=of w, label=above:$w_1 : v\underline{l}r\underline{sd}$ ] {};
\node[point] (v1) [right=of v, label=below:$w_2 : v\underline{lrsd}$, label={[label distance=10mm]-90:$\cM_1$}] {};

\node[point] (w2) [right=of w1,label=below:$w_1 : vlrs\underline{d}$] {};
\node[point] (v2) [right=of v1, label=below:$w_2 : vl\underline{r}s\underline{d}$ , label={[label distance=10mm]-90:$\cM_2$}] {};

\node (w3) [right=of w2] {};
\node[point] (v3) [right=of v2, label=below:$w_2 : vl\underline{r}sd$ , label={[label distance=10mm]-90:$\cM_{32}$}] {};

\node (w4) [right=of w3] {};
\node[point] (v4) [right=of v3, label=below:$w_2 : \underline{v}l\underline{r}sd$, label={[label distance=10mm]-90:$\cM_{42}$}] {};

\path[-] (w) edge (v);
\path[-] (w1) edge (v1);

\path[->,dashed] (w) edge (w1);
\path[->,dashed] (v) edge (v1);

\path[->,dashed] (w1) edge (w2);
\path[->,dashed] (v1) edge (v2);

\path[->,dashed] (v2) edge (v3);
\path[->,dashed] (v3) edge (v4);

\node[draw=black, rounded corners, fit={($(w)+(-30pt,30pt)$) ($(v)+(30pt,-60pt)$)}](FIt) {};
\node[draw=black, rounded corners, fit={($(w1)+(-30pt,30pt)$) ($(v1)+(30pt,-60pt)$)}](FIt1) {};
\node[draw=black, rounded corners, fit={($(w2)+(-30pt,30pt)$) ($(v2)+(30pt,-60pt)$)}](FIt2) {};
\node[draw=black, rounded corners, fit={($(w3)+(-30pt,30pt)$) ($(v3)+(30pt,-60pt)$)}](FIt3) {};
\node[draw=black, rounded corners, fit={($(w4)+(-30pt,30pt)$) ($(v4)+(30pt,-60pt)$)}](FIt4) {};

\node  at ($0.5*(v)+0.5*(v1)+(0,-80pt)$) (a1) {$\move$};
\node  at ($0.5*(v1)+0.5*(v2)+(0,-80pt)$) (a1) {$\flick$};
\node  at ($0.5*(v2)+0.5*(v3)+(0,-80pt)$) (a1) {$\takel$};
\node  at ($0.5*(v3)+0.5*(v4)+(0,-80pt)$) (a1) {$\move$};
\end{tikzpicture}
\vspace{-0.75cm}
\end{center}
	\caption{Execution of the sequence of actions $\tuple{\move,\flick,\takel,\move}$ starting at $\cM_0,w_2$ of Example~\ref{ex:pink}. This figure is analogous to Figure~\ref{fig:pink.temporal-r} but replacing action $\taker$ by $\takel$ and models $\cM_{31}$ and $\cM_{41}$ by $\cM_{32}$ and $\cM_{42}$, respectively.
	 }\label{fig:pink.temporal-l}
\end{figure}

To conclude this section,
we formalise the concepts of planing task and  planning solution.

\begin{definition}[Planning task]
Given the disjoint sets of actions~\mbox{$\cA \subseteq \at$}
and fluents~\mbox{$\cF \subseteq \at$},
a \emph{planning task}
is a triple $\Pi = \tuple{\Gamma_0,\Gamma,\varphi_g}$
where $\Gamma_0$ is a theory over $\at \setminus \cA$ defining the initial state,
$\Gamma$ is a theory over $\atbi$ defining the interpretation of actions
and $\varphi_g$ is the \emph{goal formula} over~$\cF$.\qed
\end{definition}

\begin{definition}[Planning solution]
A plan $\pi$ is a \emph{conditional solution} for the planning task $\Pi = \tuple{\Gamma_0,\Gamma,\varphi_g}$ iff $\model{\Gamma_0,\cF} \models \den{\pi}^\Gamma \varphi_g$.
A conditional solution without occurrences of the ``$\iif$-$\tthen$-$\eelse$'' construct is called a \emph{conformant solution}.\qed
\end{definition}

In particular, Example~\ref{ex:pink} can be now formalised as the planning task
$\Pi = \tuple{\Gamma_0,\GammaPink,\varphi_g}$
where $\Gamma_0$ is a singleton containing  $\eqref{eq:formula.state0}$, describing the initial situation, and $\varphi_g = \sneg v \wedge d$.
Then, we can see that~\eqref{eq:conditional.plan}
is a conditional solution for the planning task~$\Pi$.
We can also formalise Example~\ref{ex:pink.simple} as the planning task
$\Pi = \tuple{\Gamma_0',\GammaPink,\varphi_g}$
where
$\Gamma_0'$ contains the single formula:
\begin{gather}
\varphi_0' = \sneg v \wedge \sneg l \wedge r \wedge \sneg s \wedge \sneg d
	\label{eq:simple.formula.state0}
\end{gather}
that describes the corresponding initial situation.
It can be checked that~\eqref{eq:conditional.plan} is also a conditional solution for $\Pi'$, though this example also has the simpler (conformant) solution:
\begin{gather}
\move\;;\, \taker \;;\, \move 
	\label{eq:simple.plan}
\end{gather}
Finally, Example~\ref{ex:pink.try} becomes the task
$\Pi = \tuple{\Gamma_0,\GammaPink',\varphi_g}$ for which
\begin{gather}
\move\;;\, \ttaker \;;\, \ttakel \;;\, \move 
	\label{eq:conformant.plan}
\end{gather}
is a conformant solution.

%% file: discussion.tex

\section{Conclusions and Future Work}

As discussed in~\cite{Bolander17}, the traditional DEL[$\cE$] approach with event model updates is a \emph{semantic approach}, where states and actions are represented as semantic objects, epistemic and event models respectively.
On the other hand, DEL[ASP] is a \emph{syntactic approach}, where states and actions are represented as knowledge-bases, that is, sets of formulas known to be true.
Semantic and syntactic approaches are mutually dual, with the semantic approach modelling \emph{ignorance} (the more ignorance, the bigger the state) and the syntactic approach modelling \emph{knowledge} (the more knowledge, the bigger the knowledge-base).
The generalisation of DEL for abstract updating objects can easily accommodate both approaches: it suffices with allowing both event models and epistemic programs to occur in the dynamic operator, and selecting the corresponding updating evaluation.

Another interesting observation is that both DEL[$\cE$] and ASP can be considered as
generalisations of the STRIPS planning language
in orthogonal directions.
On the one hand,
DEL[$\cE$] allows planning in domains where the world is not fully observable, the effects of actions are not necessarily deterministic and where sensing actions may allow to gain knowledge about the actual state of the world.
On the other hand, ASP introduces high level KR features like the treatment of indirect effects, action qualifications, state constraints or recursive fluents (for motivation about the need of such features we refer to~\cite{lifschitz2002}). 
The approach presented here, DEL[ASP], combines the strengths of both generalisations so that it is possible to use high level KR features in non-fully observable or non-deterministic domains where observing the world may be needed to achieve a valid plan.

Similar to our approach, the action language $m\mathcal{AL}$~\cite{baral2013} also combined the treatment of indirect effects and action qualifications with the possibility of defining sensing actions.
The main handicap of $m\mathcal{AL}$ with respect to DEL[ASP] is that the former only allows ramifications on the ontic effects, but not on the epistemic ones, as we did for instance with rule~\eqref{eq:obs}.
In $m\mathcal{AL}$, as in DEL[$\cE$], this indirect observation needs to be encoded as a direct effect of all actions that may affect those fluents. On the other hand, an advantage of both
DEL[$\cE$] and $m\mathcal{AL}$ is that they can be applied
on domains that involve several agents and in which those agents may even hold false beliefs~\cite{Bolander14},
while, so far, DEL[ASP] is only able to deal with domains involving a single agent.
Extending DEL[ASP] to cover these domains is a matter of future work.
It will be also interesting to study the relation between DEL[ASP] and Temporal ASP~\cite{cakascsc18a} and the possibility of extending the latter with an epistemic modality to deal with non-fully observable or non-deterministic domains.

Regarding the computation of planning solutions in DEL[ASP], it is worth to mention that the algorithm based on planning trees described in~\cite{AndersenBJ12} for DEL[$\cE$] is general enough and does not really depend of the kind of updating object used.
In this sense, we can apply that same algorithm with the only variation of using the ASP updating evaluation when we expand the tree.
Then, solutions can be retrieved from the planning tree in exactly the same way as described there.

\paragraph{Acknowledgements} This work has been partially funded by the Centre International de Math\'{e}matiques et d'Informatique de Toulouse through contract ANR-11-LABEX-0040-CIMI within the program ANR-11-IDEX-0002-02, grant 2016-2019 ED431G/01 CITIC Center (Xunta de Galicia, Spain), grant TIN 2017-84453-P (MINECO, Spain).